\documentclass[sigconf]{aamas}

\usepackage{balance} 
\usepackage{graphicx}
\usepackage{enumitem}

\usepackage{subcaption}

\newtheorem*{remark}{Remark}

\usepackage{algorithm}
\usepackage{algpseudocode}

\newtheorem{theorem}{Theorem}[section]
\newtheorem{definition}{Definition}[section]

\newtheorem{lemma}{Lemma}[section]

\newtheorem{assumption}{Assumption}

\algnotext{EndFor}

\makeatletter
\gdef\@copyrightpermission{
  \begin{minipage}{0.2\columnwidth}
   \href{https://creativecommons.org/licenses/by/4.0/}{\includegraphics[width=0.90\textwidth]{by.pdf}}
  \end{minipage}\hfill
  \begin{minipage}{0.8\columnwidth}
   \href{https://creativecommons.org/licenses/by/4.0/}{This work is licensed under a Creative Commons Attribution International 4.0 License.}
  \end{minipage}
  \vspace{5pt}
}
\makeatother

\setcopyright{ifaamas}
\acmConference[AAMAS '25]{Proc.\@ of the 24th International Conference
on Autonomous Agents and Multiagent Systems (AAMAS 2025)}{May 19 -- 23, 2025}
{Detroit, Michigan, USA}{Y.~Vorobeychik, S.~Das, A.~Nowé  (eds.)}
\copyrightyear{2025}
\acmYear{2025}
\acmDOI{}
\acmPrice{}
\acmISBN{}

\acmSubmissionID{709}

\title[FedRLHF]{FedRLHF: A Convergence-Guaranteed Federated Framework for
Privacy-Preserving and Personalized RLHF}

\author{Flint Xiaofeng Fan}
\affiliation{
  \institution{National University of Singapore}
  \country{Singapore}}
\email{fxf@u.nus.edu}

\author{Cheston Tan}
\affiliation{
  \institution{CFAR, A*STAR}
  \country{Singapore}}
\email{cheston-tan@i2r.a-star.edu.sg}

\author{Yew-Soon Ong}
\affiliation{
  \institution{CFAR, A*STAR}
  \country{Singapore}}
\email{asysong@ntu.edu.sg}

\author{Roger Wattenhofer}
\affiliation{
  \institution{ETH Zurich}
  \city{Zurich}
  \country{Switzerland}}
\email{wattenhofer@ethz.ch}

\author{Wei-Tsang Ooi}
\affiliation{
  \institution{National University of Singapore}
  \country{Singapore}}
\email{ooiwt@comp.nus.edu.sg}

\begin{abstract}
In the era of increasing privacy concerns and demand for personalized experiences, traditional Reinforcement Learning with Human Feedback (RLHF) frameworks face significant challenges due to their reliance on centralized data. We introduce Federated Reinforcement Learning with Human Feedback (FedRLHF), 
a novel framework that decentralizes the RLHF process.
FedRLHF enables collaborative policy learning across multiple clients, such as Large Language Models (LLMs) finetuning, without sharing raw data or human feedback, thereby ensuring robust privacy preservation.
Leveraging federated reinforcement learning, each client integrates human feedback locally into reward functions and updates their policies through personalized RLHF processes. 
We establish rigorous theoretical foundations for FedRLHF, providing convergence guarantees, and deriving sample complexity bounds that scale efficiently with the number of clients. 
Empirical evaluations on the MovieLens and IMDb datasets demonstrate that FedRLHF preserves user privacy, achieves performance on par with centralized RLHF, and enhances personalization across diverse client environments. 
\end{abstract}


\newcommand{\BibTeX}{\rm B\kern-.05em{\sc i\kern-.025em b}\kern-.08em\TeX}

\begin{document}


\pagestyle{fancy}
\fancyhead{}


\maketitle 


\section{Introduction}
Reinforcement Learning with Human Feedback (RLHF) has emerged as a powerful paradigm for training intelligent agents that align closely with human values and preferences~\cite{christiano2017deep, ouyang2022training}. By integrating human feedback into the reinforcement learning loop, RLHF has enabled significant advancements in natural language processing, robotics, and personalized recommendation systems~\cite{stiennon2020learning,zou2019reinforcement,corecco2024llm}. A prominent example is ChatGPT~\cite{openai2023chatgpt}, where RLHF has been instrumental in fine-tuning large language models (LLMs) to generate more coherent, contextually appropriate, and user-aligned responses~\cite{ziegler2019fine}.
\begin{figure}[t]
    \centering
    \includegraphics[width=0.92\linewidth]{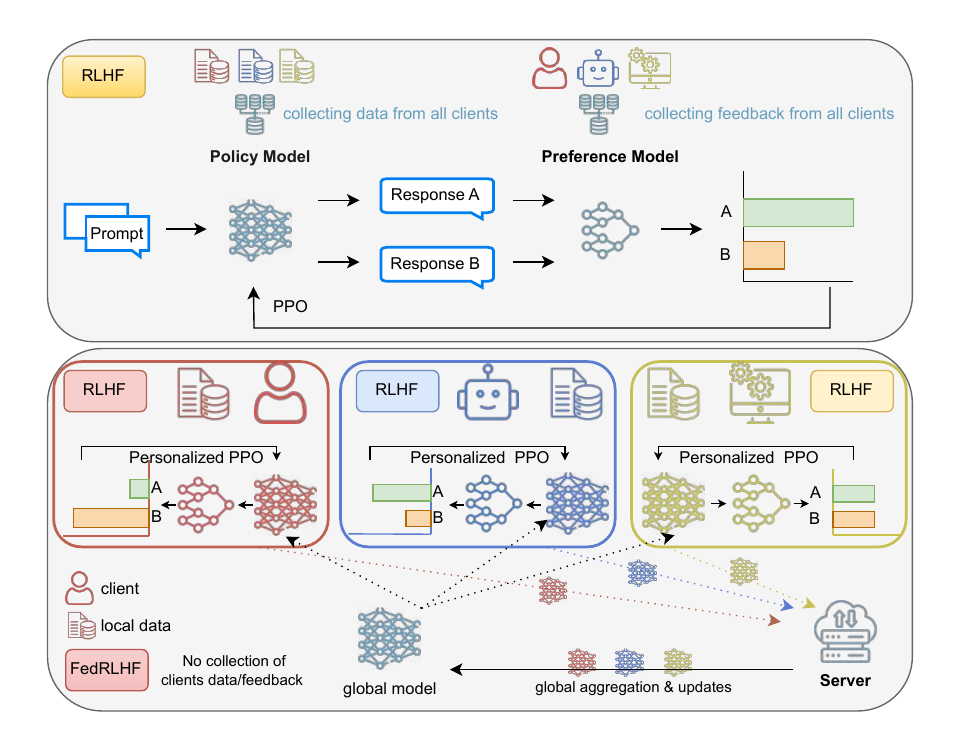}
\caption{
Comparison of the FedRLHF framework to conventional RLHF methods. Top: Conventional RLHF requires centralized collection of user data and feedback to train the policy model and preference (reward) model. Bottom: In FedRLHF, clients maintain local policy models trained on-device using RLHF with local data and preference models. Only policy model updates are shared with a central server, which aggregates them to refine a global policy.
}
\Description{Comparison of the FedRLHF framework to conventional RLHF methods.}
    \label{fig:problem_setting}
\end{figure}

Despite these successes,
the practice of aggregating data and feedback from multiple users 
in centralized RLHF systems poses significant \emph{privacy} risks, especially in domains involving sensitive information such as healthcare or finance~\cite{shokri2015privacy}. 
For instance, a personalized healthcare assistant using RLHF requires centralizing patient data and feedback, potentially violating regulations like Health Insurance Portability and Accountability Act (HIPAA)~\cite{hipaa1996act} and General Data Protection Regulation (GDPR)~\cite{GDPR2016},
while exposing users to risks of data breaches and identity theft.
%
Moreover, different organizations or individuals may be reluctant to share their feedback data due to intellectual property concerns or competitive advantages~\cite{li2020federated}.

In addition to privacy risks, centralization creates substantial hurdles for achieving \emph{personalization} in RLHF systems. Users exhibit diverse preferences and behaviors, making it challenging for a centralized policy to cater effectively to all individuals~\cite{christiano2017deep}. Balancing global performance with personalization becomes non-trivial, as optimizing for the average user can lead to suboptimal experiences for specific individuals~\cite{smith2017federated}. Returning to our healthcare assistant example, patients may have unique health conditions and treatment preferences. A one-size-fits-all model, trained on centralized data, may fail to provide the personalized recommendations necessary for optimal care, potentially impacting patient outcomes negatively.

To address these limitations, we propose \textbf{Federated Reinforcement Learning with Human Feedback (FedRLHF)}, a novel federated framework that uniquely integrates RLHF principles with federated reinforcement learning~\cite{my-pub:fedpg-br,my-pub:DAI2024257} to simultaneously address privacy and personalization challenges. 
As illustrated in Figure~\ref{fig:problem_setting},
FedRLHF decentralizes the RLHF process so that each client updates a local policy using only its own data and private human feedback. By exchanging only model updates rather than raw data, FedRLHF preserves privacy and adheres more closely to data governance regulations. 
Moreover, each client can personalize its learning process by shaping the local reward function with individual human feedback, ensuring that policies align with user preferences. 
This decentralized approach supports collaboration across clients while maintaining privacy and personalization.


The FedRLHF framework introduces unique technical challenges, notably ensuring \emph{convergence guarantees} in a decentralized environment
and managing the \emph{trade-off between global performance and personalization}. In federated reinforcement learning, the variability in clients' environments and behaviors can lead to instability and divergence in the learning process.
Our FedRLHF framework addresses these challenges in a principled manner. We provide a comprehensive theoretical analysis establishing convergence guarantees and deriving bounds on the sample complexity under standard assumptions. 
%
Furthermore, we introduce a quantitative measure of personalization and a {personalization parameter} in the reward shaping function to analyze and control the trade-off between 
global policy alignment
and
personalized adaptation.

Our contributions are as follows:
\begin{itemize} 
    \item We introduce FedRLHF, a framework that integrates federated reinforcement learning with human feedback, enabling privacy-preserving and personalized policy learning and model fine-tuning (Section~\ref{sec:problem_formulation}).
    \item We provide convergence guarantees and derive sample complexity bounds that account for the integration of human feedback, extending existing FedRL theory (Section~\ref{sec:convergence-results}).
    \item We develop a quantitative measure of personalization to analyze the trade-off between maximizing global performance and adapting individual client policies (Section~\ref{sec:personalization_analysis}).
    \item We empirically demonstrate FedRLHF's effectiveness on the MovieLens and IMDb datasets, showcasing its ability to preserve privacy, match centralized RLHF performance, and enhance personalization (Section~\ref{sec:experimental_results}).
    \item A full version of this work, including additional proofs and experimental details, is available at \cite{fan2024fedrlhf}, and our code is publicly available at \href{https://github.com/flint-xf-fan/Federated-RLHF}{github.com/flint-xf-fan/Federated-RLHF}.
\end{itemize}

\section{Background \& Related Work}\label{sec:related_work}
\textbf{Reinforcement Learning with Human Feedback} (RLHF) has become instrumental in aligning machine learning models with human values and preferences \cite{christiano2017deep, ouyang2022training,ziegler2019fine,stiennon2020learning}. By integrating human-generated feedback into the reward structure, RLHF facilitates the training of policies that exhibit behaviors more aligned with human preferences. 
However, traditional RLHF frameworks predominantly rely on centralized aggregation of data and feedback, which introduces significant privacy concerns.
In addition, 
the centralization nature of RLHF also limits personalization, as it typically involves training a single reward model for all clients. 
This approach fails to accommodate the heterogeneous preferences of individual clients, leading to suboptimal policy performance in diverse environments, as illustrated in Figure~\ref{fig:problem_setting}.

While previous studies~\cite{park2024principled} have explored personalized reward models in RLHF through representation learning, they still depend on the centralized aggregation of user data and feedback, failing to address the privacy concerns inherent in centralized systems. 
Additionally, methods like Group Robust Preference Optimization (GRPO)~\cite{ramesh2024group} aim to reduce performance disparities across user groups but similarly rely on aggregated, centralized datasets and do not provide mechanisms for personalization at the individual level. 
A recent effort by Li et al. \cite{li2024personalized} introduces a framework for personalized language modeling from personalized human feedback, which learns user-specific embeddings to capture individual preferences. However, their approach still relies on centralized data collection and does not address privacy concerns. 
Moreover, these prior approaches do not simultaneously tackle both privacy preservation and personalization challenges in real-world scenarios where user data is distributed across multiple devices or organizations.

%
%
\textbf{Federated Reinforcement Learning} (FedRL) has garnered significant attention in recent years, aiming to leverage the principles of federated learning~\cite{mcmahan2017communication} across diverse RL clients to enhance their sample efficiency without sharing raw data or trajectories of the sequential decision-making process. 
This approach has shown promise in diverse applications, from optimizing autonomous vehicles and enhancing edge caching in IoT networks to smart management of building facibilities~\citep[etc.]{my-pub:dai2023federated,wang2020federated-FedRL-5,FedRL-building,liu2019FedRL-for-robots,yu2020FedRLfor5G,my-pub:caesar}.
Recent theoretical advancements have solidified the foundations of FedRL. Notably, convergence guarantees and sample complexity bounds have been established, demonstrating speedup with increasing numbers of participating agents \cite{my-pub:fedpg-br}. In addition, the application of Markovian sampling techniques has been shown to achieve linear convergence speedup in FedRL settings \cite{frl-linear}. Furthermore, recent analysis of decentralized FedRL with Byzantine fault tolerance has proven fast convergence rates without relying on a central server, marking a significant step towards fully distributed and resilient RL systems \cite{my-pub:decbyzpg,10621347}.
Recent works have explored the benefits of heterogeneity among Q-learning agents~\cite{my-pub:fedhql-arxiv}. Woo et al.~\cite{frl-pmlr-v202-woo23a} prove that leveraging agent heterogeneity can lead to linear speedup and significant efficiency gains in 
federated Q-learning settings. 
In addition, Woo et al.~\cite{woo2024federated} introduce a federated offline RL method that achieves linear speedup with low communication costs 
in heterogeneous client environments.
Moreover,
Wang et al.~\cite{wang2024momentum} leverage momentum mechanisms to achieve exact convergence and state-of-the-art sample efficiency in highly heterogeneous environments.
Finally, Jiang et al.~\cite{jiang2025fedhpd} propose the first policy distillation-based framework that aligns heterogeneous agent policies to further accelerate convergence and enhance sample efficiency in federated policy gradient methods.


\textbf{Novelty of Our Approach.} 
In contrast to existing FedRL methods that employ a uniform reward structure, our approach, FedRLHF, integrates client-specific human feedback directly into the reward function. 
By allowing each client to locally shape its reward,
FedRLHF not only preserves privacy (since sensitive feedback remains on-device), but also enhances personalization by accommodating diverse user preferences. 
This personalized feedback mechanism is particularly transformative for LLM applications (e.g., ChatGPT), as it enables context-aware fine-tuning of responses while upholding stringent privacy standards.
Moreover, our framework is theoretically distinguished by convergence guarantees and sample complexity bounds that explicitly capture the variability introduced by human feedback—an aspect absent from prior FedRL literature.

\section{Problem Formulation}
\label{sec:problem_formulation}
We consider a federated reinforcement learning system with $K$ clients, where each client $k \in {1, 2, \dots, K}$ interacts with its own environment, modeled as a Markov Decision Process (MDP) $M_k = (\mathcal{S}, \mathcal{A}, P_k, R_k,\rho_0(s), \gamma)$. Here, $\mathcal{S}$ and $\mathcal{A}$ represent the state and action spaces, respectively. The state-transition function $P_k: \mathcal{S} \times \mathcal{A} \times \mathcal{S} \to [0,1]$ defines the probability $P_k(s' \mid s, a)$ of transitioning from state $s$ to state $s'$ after taking action $a$. The reward function $R_k: \mathcal{S} \times \mathcal{A} \to \mathbb{R}$ specifies the expected reward $R_k(s, a)$ for client $k$ when action $a$ is taken in state $s$. Finally, $\gamma \in [0,1)$ is the discount factor balancing immediate and future rewards.
The initial state distribution \( \rho_0(s) \) specifies the probability of the MDP starting in state \( s \). We assume that both \(\gamma\) and \( \rho \) are fixed and known for all clients. This uniformity simplifies the theoretical analysis and ensures consistency in policy evaluation and optimization. However, extensions to client-specific discount factors and initial state distributions are straightforward within our framework.

Each client's MDP may vary in the transition dynamics $P_k$ and reward functions $R_k$, reflecting the heterogeneity among clients due to personalized environments or preferences. Let $\pi_\theta: \mathcal{S} \to \Delta(\mathcal{A})$ denote a stochastic policy parameterized by $\theta \in \mathbb{R}^d$, where $\Delta(\mathcal{A})$ is the set of probability distributions over the action space $\mathcal{A}$. The policy $\pi_\theta(a \mid s)$ specifies the probability of taking action $a$ in state $s$ under the parameters $\theta$. For each client $k$, the objective is to find the policy parameters $\theta$ that maximize the expected cumulative discounted reward:
\begin{equation}
\label{eq:client-objective}
J_k(\theta) = \mathbb{E}_{\tau \sim \pi_\theta} \left[ \sum_{t=0}^\infty \gamma^t R_k(s_t, a_t) \right],
\end{equation}
where the expectation is taken over trajectories $\tau = (s_0, a_0, s_1, \ldots)$ generated by following policy $\pi_\theta$ in $M_k$
, starting from $\rho_o(s)$.

\subsection{Incorporating Human Feedback}
Unlike conventional RLHF where a single reward function and policy are learned from aggregated data, FedRLHF allows for client-specific reward functions $R_k$ and locally adapted policies $\pi_k$. 
In Fed\-RLHF, human feedback is integrated locally to shape each client's reward function. 
Specifically, the reward function for client $k$ is augmented as:
\begin{equation}
\label{eq:reward-shaping}
R_k(s, a) = R_k^0(s, a) + \lambda H_k(s, a),
\end{equation}

where $R_k^0(s, a)$ is the intrinsic reward provided by the environment, $H_k(s, a)$ is the client-specific human feedback function representing additional reward or penalty based on human evaluation, and $\lambda > 0$ is a scaling factor balancing the influence of human feedback relative to the intrinsic reward.

\subsection{Global Objective}

The global objective is to find policy parameters $\theta$ that maximize the average expected cumulative reward across all clients:

\begin{equation}
\label{eq:global-objective}
J(\theta) = \frac{1}{K} \sum_{k=1}^K J_k(\theta).
\end{equation}


\subsection{The FedRLHF Framework}
\label{subsec:fedrlhf_algorithm}

We propose the \textbf{Federated Reinforcement Learning with Human Feedback (FedRLHF)} framework to optimize the global objective \eqref{eq:global-objective} across all clients in a federated manner while respecting their individual environments and preferences. 
Algorithm~\ref{alg:fedrlhf} presents the pseudocode for FedRLHF. The key components and operations of the framework are as follows:

\begin{algorithm}[t]
\caption{FedRLHF}
\label{alg:fedrlhf}
\begin{algorithmic}[1]
\Require Number of clients $K$, total communication rounds $T$, local update steps $\tau$, personalization factor $\lambda$, learning rate $\eta$
\Ensure Final global policy parameters $\theta_{\text{final}}$
\State Initialize global policy parameters $\theta_0$
\For{$t = 0$ \textbf{to} $T-1$}
    \State \textbf{Server} broadcasts global parameters $\theta_t$ to all clients
    \For{\textbf{each client} $k \in \{1, 2, \dots, K\}$ \textbf{in parallel}}
        \State Initialize local parameters: $\theta_{t,0}^k \gets \theta_t$
        \For{$i = 0$ \textbf{to} $\tau - 1$}
            \State Sample a mini-batch $\mathcal{B}_{t,i}^k$ using policy $\pi_{\theta_{t,i}^k}$ in $M_k$
            \State Collect human feedback $H_k(\mathcal{B}_{t,i}^k)$
            \State Calculate shaped reward $R_k$ {per Eq.~\eqref{eq:reward-shaping}}
            \State Estimate policy gradient per ${J}_k(\theta^k_{t,i})$ in Eq.~\eqref{eq:client-objective}:
            \begin{equation*}
            \hat{g}_{t,i}^k \gets \nabla_{\theta} \hat{J}_k(\theta_{t,i}^k; \mathcal{B}_{t,i}^k) 
            \end{equation*}
            \State Update local parameters:
            \begin{equation*}
            \theta_{t,i+1}^k \gets \theta_{t,i}^k + \eta \hat{g}_{t,i}^k
            \end{equation*}
        \EndFor
        \State Compute local model update:
        \begin{equation*}
        \Delta \theta_{t+1}^k \gets \theta_{t,\tau}^k - \theta_t
        \end{equation*}
    \EndFor
    \State \textbf{Server} aggregates local updates:
    \begin{equation*}
    \theta_{t+1} \gets \theta_t + \frac{1}{K} \sum_{k=1}^K \Delta \theta_{t+1}^k
    \end{equation*}
\EndFor
\State \Return $ \theta_{\text{final}} \gets \theta_T$ or $\frac{1}{T}\sum_{t=0}^{T-1} \theta_t$
\end{algorithmic}
\end{algorithm}

     \paragraph{Local RLHF (lines 6-12)}: The FedRLHF framework allows clients to use different RL methods, including Q-learning \cite{watkins1989learning-Q-learning} and policy gradient (PG)~\cite{williams1992REINFORCE}, to perform $\tau$ steps of local optimization.
    The theoretical analysis provided in
    this manuscript
    assumes PG methods for local updates, which was necessary to facilitate the convergence analysis and derive sample complexity bounds. 
     \paragraph{Reward Shaping (line 9)}: Rewards are shaped as $R_k = R_k^0 + \lambda H_k$, where $R_k^0$ is the intrinsic reward and $\lambda$ controls the influence of human feedback $H_k$.
     \paragraph{Privacy Preservation}: 
    Only the model updates $\Delta \theta_{t+1}^k$ are transmitted to the server, preventing direct access to individual user data and feedback, thus providing a significant level of privacy protection. 
    For applications requiring formal privacy guarantees or protection against advanced inference attacks, additional measures such as Differential Privacy (DP)~\cite{dwork2006calibrating} could be incorporated into the framework. 
     \paragraph{Aggregation (line 13)}:
    Multiple server aggregation methods exist, such as FedAvg (simple averaging) \cite{mcmahan2017communication}, Weighted Average based on client data sizes, and robust aggregation techniques like median-based methods \cite{yin2018byzantine}. The choice depends on specific requirements such as fairness, robustness to outliers, or heterogeneity in client data. In this algorithm and our theoretical analysis, we employ FedAvg for its simplicity and to facilitate clearer theoretical results. However, our framework is flexible and can accommodate other aggregation methods if needed.

\paragraph{Mechanism of Personalization.}
Although the server broadcasts a single global policy parameter vector, personalization arises because each client’s local objective includes its own reward function, shaped by both intrinsic rewards and the local human feedback function \(H_k\). At each round, client \(k\) adapts the global parameters to better optimize \(\lambda H_k\) alongside \(R_k^0\). This leads to client-specific parameter updates. While the global model aggregates these updates to improve overall performance, each client’s environment and feedback distribution remain unique. In practice, clients can also maintain partially fine-tuned local parameters or personalized embeddings, thus capturing their distinct preferences or constraints. Consequently, even though the server aggregates all updates, the local RLHF step ensures that each client’s policy is guided by its private human feedback, achieving personalization in a federated manner.
\section{Convergence Results}
\label{sec:convergence-results}

In this section, we provide theoretical guarantees on the convergence and sample complexity of the FedRLHF framework. 
These results apply to implementations of the framework that adhere to the core principles and structure outlined in Algorithm~\ref{alg:fedrlhf}, under the necessary assumptions stated below:

\subsection{Assumptions}

\begin{assumption}[$L$-smooth gradients]\label{assumption:L-smoothness}
For all $\theta, \theta' \in \mathbb{R}^d$ and $k \in [K]$, the gradients of the clients' objective functions are $L$-Lipschitz continuous:
\[
\|\nabla J_k(\theta) - \nabla J_k(\theta')\| \leq L\|\theta - \theta'\|.
\]
\end{assumption}

\begin{assumption}[$G$-bounded gradients]\label{assumption:G-bounded-gradients}
For all $\theta \in \mathbb{R}^d$ and $k \in [K]$, the gradients are bounded:
\[
\|\nabla J_k(\theta)\| \leq G.
\]
\end{assumption}

\begin{assumption}[$\sigma$-bounded variance]\label{assumption:bounded-variance}
For all $\theta \in \mathbb{R}^d$ and $k \in [K]$, the variance of the stochastic gradient estimator is bounded:
\[
\mathbb{E}\left[\|\nabla \hat{J}_k(\theta) - \nabla J_k(\theta)\|^2\right] \leq \sigma^2,
\]
where $\nabla \hat{J}_k(\theta)$ is the stochastic gradient computed from a mini-batch.
\end{assumption}

\begin{assumption}[Bounded second moment]\label{assumption:bounded-second-moment}
For all $\theta \in \mathbb{R}^d$ and $k \in [K]$, the second moment of the stochastic gradient is bounded:
\[
\mathbb{E}\left[\|\nabla \hat{J}_k(\theta)\|^2\right] \leq M^2.
\]
\end{assumption}

\begin{assumption}[Polyak-Łojasiewicz (PL) condition]\label{assumption:PL-condition}
The global objective function satisfies the PL condition:
\[
2\mu\left(J(\theta^*) - J(\theta)\right) \leq \|\nabla J(\theta)\|^2, \quad \forall \theta \in \mathbb{R}^d,
\]
where $\mu > 0$ is a constant and $\theta^* = \arg\max_{\theta} J(\theta)$.
\end{assumption}


\begin{remark}
    Assumptions \ref{assumption:L-smoothness}--\ref{assumption:bounded-second-moment} are common in the stochastic optimization literature. 
    The PL condition (Assumption~\ref{assumption:PL-condition}) is stronger, especially for reinforcement learning's typically non-convex objectives. 
    However, it approximates scenarios where objective functions exhibit properties conducive to linear convergence.
    Policy gradient methods with trust region constraints, such as TRPO~\cite{schulman2015TRPO}, or those using proximal objectives, like PPO~\cite{schulman2017proximal-PPO}, often result in smoother updates to the policy parameters, making the PL condition more reasonable.
    Recent works~\cite[etc.]{karimi2016linear,yuan2022general,bhandari2024global} have used the PL condition for non-convex convergence guarantees for RL, further justifying its use in our analysis.
\end{remark}

\begin{assumption}[Bounded Human Feedback]\label{assumption:bounded-human-feedback}
For all $s \in \mathcal{S}$, $a \in \mathcal{A}$, and $k \in [K]$, the human feedback is bounded:
\[
|H_k(s, a)| \leq H_{\max}.
\]
\end{assumption}

\begin{remark}
Assumption~\ref{assumption:bounded-human-feedback} limits the variance introduced by human feedback in the learning process.  
In our experiments with the MovieLens task, 
we implement this by bounding feedback values and options (Section~\ref{subsec:hf-simulation}), similar to practical systems like ChatGPT that curate feedback for consistency.

\end{remark}

\subsection{Convergence and Sample Complexity}

We now present the main theoretical results, starting with key lemmas leading up to the convergence theorem.
The complete proof for theorems presented in this section is provided in Appendix~\ref{appendix:sec:proofs_theorems_convergence}.

\begin{lemma}[Bounded Local-Global Difference]\label{lemma:bounded-local-global-difference}
Under Assumptions~\ref{assumption:L-smoothness}, \ref{assumption:G-bounded-gradients}, and \ref{assumption:bounded-variance}, for any communication round $t$ and client $k$, we have:
\[
\mathbb{E}\left[\|\theta_t^k - \theta_t\|^2\right] \leq \eta^2\tau^2(G^2 + \sigma^2)
\]
where $\theta_t^k$ is the local model of client $k$, $\theta_t$ is the global model, $\eta$ is the learning rate, $\tau$ is the number of local updates, $G$ is the gradient bound, and $\sigma^2$ is the variance bound.
\end{lemma}

\begin{remark}
Lemma~\ref{lemma:bounded-local-global-difference} quantifies the extent to which local models diverge from the global model after $\tau$ local updates. This deviation is influenced by the learning rate $\eta$ and the number of local updates $\tau$, both of which amplify the divergence when increased. 
\end{remark}

\begin{lemma}[One-Step Descent]\label{lemma:one-step-descent}
Under Assumptions~\ref{assumption:L-smoothness}--\ref{assumption:bounded-human-feedback}, for any round $t$, the expected improvement in the global objective satisfies:
\begin{align*}
\mathbb{E}[J(\theta_{t+1})] \geq J(\theta_t) & + \eta\tau \left(1 - \frac{L \eta \tau}{2}\right) \|\nabla J(\theta_t)\|^2  \\
& - \frac{L}{2} \left(\frac{\eta^2 \tau^2}{K}\right) (G^2 + \sigma^2) - \lambda H_{\max}
\end{align*}
where $\theta_{t+1}$ is the updated global model.
\end{lemma}

\begin{remark}
    This lemma establishes that each communication round in FedRLHF yields a quantifiable improvement in the global objective \( J(\theta) \). The positive term \( \eta\tau\left(1 - \frac{L\eta\tau}{2}\right)\|\nabla J(\theta_t)\|^2 \) signifies progress towards maximizing the objective, while the negative terms \( \lambda H_{\max} \) and \( \frac{L}{2}\left(\frac{\eta^2\tau^2}{K}\right)(G^2 + \sigma^2) \) account for the inherent variance in stochastic gradients and the bounded impact of human feedback, respectively.
\end{remark}

\begin{theorem}[Convergence of FedRLHF]\label{theorem:convergence}
Under Assumptions~\ref{assumption:L-smoothness}--\ref{assumption:bounded-human-feedback}, if we choose the constant learning rate $\eta = \frac{1}{L\tau}$,
then the output \(\theta_{\text{avg}} = \frac{1}{T} \sum_{t=0}^{T-1} \theta_t\) of Algorithm~\ref{alg:fedrlhf} satisfies:
\begin{align*}
\mathbb{E}\left[ J(\theta^*) - J(\theta_{\text{avg}}) \right] &\leq \frac{L}{\mu T} \left( J(\theta^*) - J(\theta_0) \right) \\
 & \quad + \frac{1}{2\mu K}(G^2 + \sigma^2) + \frac{L}{\mu}\lambda H_{\max}.
\end{align*}
\end{theorem}


Theorem~\ref{theorem:convergence} establishes that the FedRLHF algorithm converges to an optimal policy within a bounded suboptimality gap. This bound elucidates several key aspects of the algorithm's performance:

\begin{enumerate}[leftmargin=*]
    \item \textbf{Convergence Rate (\( \mathcal{O}(1/T) \)):} The first term \( \frac{L}{\mu T} \left( J(\theta^*) - J(\theta_0) \right) \) 
    indicates that the algorithm achieves a linear convergence rate with respect to the number of communication rounds $T$, which aligns well with expectations in federated optimization under the PL condition.
    \item \textbf{Impact of Client Diversity and Variance (\( \mathcal{O}(1/K) \)):} The second term \( \frac{1}{2\mu K}(G^2 + \sigma^2) \) scales inversely with the number of clients \( K \). This indicates that aggregating updates from more clients reduces the effect of gradient variance \( \sigma^2 \) and bounded gradient norms \( G \), leading to a tighter convergence bound. 
    %
    \item \textbf{Influence of Human Feedback (\( \mathcal{O}(1) \)):} The third term \( \frac{L}{\mu}\lambda H_{\max} \) represents the bounded influence of human feedback on the convergence. The scaling factor \( \lambda \) and the maximum human feedback bound \( H_{\max} \) determine how much human feedback can potentially offset the objective's improvement. While human feedback guides the policy towards user preferences, this term ensures that its impact remains controlled, preventing excessive deviations that could hinder convergence.
\end{enumerate}

\begin{theorem}[Sample Complexity of FedRLHF]\label{theorem:sample-complexity}
Under Assumptions~\ref{assumption:L-smoothness}--\ref{assumption:bounded-human-feedback}, to achieve an expected optimality gap of 
\[
\mathbb{E}\left[ J(\theta^*) - J(\theta_{\text{avg}}) \right] \leq \epsilon,
\]
the total number of samples required across all clients is:
\[
N = O\left( \frac{L (G^2 + \sigma^2)}{\mu^2 \epsilon^2} \right),
\]
subject to:
\[
    K \geq O\left( \frac{G^2 + \sigma^2}{\mu \epsilon} \right), \qquad     \lambda H_{\max} \leq O\left( \frac{\mu \epsilon}{L} \right).
\]
\end{theorem}


Theorem~\ref{theorem:sample-complexity} provides an estimate of the total number of samples required across all clients to achieve a desired expected optimality gap $\epsilon$. The bound, which scales with $O(\epsilon^{-2})$ and aligns with standard results in stochastic optimization, reveals several insights: 
\begin{enumerate}[leftmargin=*]
    \item \textbf{Dependence on Problem Constants:} The constants $\mu$ and $L$ reflect the curvature properties of the objective function $J(\theta)$. A larger $\mu$ (stronger PL condition) and smaller $L$ (less curvature) lead to a lower sample complexity.
 The term $G^2+\sigma^2$ captures the combined effect of the gradient bound and the variance of the stochastic gradients.
 Reducing them through techniques like gradient clipping or variance reduction methods~\cite{papini2018stochastic} can significantly lower sample complexity.
 \item \textbf{Per-Client Sample Complexity:}
Diving the bound by $K$ yields the sample complexity per client:
 $N_c= \frac{N}{K} = O\left(\frac{L (G^2 + \sigma^2)}{K \mu^2 \epsilon^2}\right)$.
 As $K$ increases, the per-client sample complexity $N_c$
  decreases proportionally, suggesting that individual clients require fewer samples to achieve the same level of accuracy when more clients participate.
This per-client sample efficiency gain aligns with the collaborative nature of federated reinforcement learing~\cite{my-pub:fedpg-br}.
    \item \textbf{Cost of Personalization:} 
    As detailed in the proof in Appendix~\ref{appendix:sec:proofs_theorems_convergence},
    a larger $\lambda H_{\max }$ increases $\epsilon_H$, effectively consuming more of the allowable error budget. 
    If $\epsilon_H$ becomes significant relative to $\epsilon$, the remaining error budget for 
    $\epsilon_T$ and $\epsilon_V$ decreases.
    This necessitates tighter convergence in these terms, requiring a larger number of communication rounds $T$ and potentially more clients $K$.
Consequently, the total sample complexity $N$ increases indirectly due to a higher emphasis on personalization, suggesting a trade-off between personalization and efficiency
\end{enumerate}

\section{Personalization-Performance Trade-off Analysis}
\label{sec:personalization_analysis}

In this section, we establish a formal relationship between personalization and global performance within the FedRLHF framework. The convergence analysis in Section~\ref{sec:convergence-results} already hints at the personalization-performance trade-off, particularly through the influence of the human feedback weight $\lambda$ on the convergence. Here, we provide a quantitative measure of this trade-off by analyzing how personalization affects global performance of intrinsic rewards and sample complexity.
The complete proof for theorems presented in this section is provided in Appendix~\ref{appendix:sec:proofs_theorem_personalization}.
\subsection{Definitions and Preliminaries}


\begin{definition}[Maximum Reward]
\label{def:Rmax}
We define \( R_{\text{max}} \) as the maximum absolute value of the intrinsic reward function across all clients and state-action pairs:
\begin{equation*}
R_{\text{max}} = \max_{k \in \{1, 2, \dots, K\},\, s \in \mathcal{S},\, a \in \mathcal{A}} |R_k^0(s, a)|.
\end{equation*}
\end{definition}

\begin{definition}[Personalization Score]
\label{def:personalization_score}
For a client \( k \) with policy \( \pi_k(\cdot|s,\theta) \) and the global policy \( \pi(\cdot|s,\theta) \), the personalization score is defined as:
\begin{equation*}
P_k(\theta) = \mathbb{E}_{s \sim \rho} \left[ D_{\text{KL}}\left(\pi_k(\cdot|s,\theta) \parallel \pi(\cdot|s,\theta)\right) \right],
\end{equation*}
where \( \rho \) is the state distribution and \( D_{\text{KL}} \) denotes the Kullback-Leibler divergence.
\end{definition}

\begin{definition}[Global Performance Metric]
\label{def:global_performance}
We define the global performance metric as the average expected cumulative intrinsic reward across all clients:
\begin{equation*}
J_g(\theta) = \frac{1}{K}\sum_{k=1}^K J_k^0(\pi),
\end{equation*}
where 
\begin{equation*}
J_k^0(\pi) = \mathbb{E}_{\tau \sim \pi} \left[ \sum_{t=0}^{\infty} \gamma^t R_k^0(s_t,a_t) \right],
\end{equation*}
and the expectation is over trajectories \( \tau = (s_0, a_0, s_1, \ldots) \) generated by following policy \( \pi \) in client \( k \)'s MDP \( M_k \), starting from \( \rho_0(s) \).
\end{definition}

\subsection{Personalization-Performance Trade-off}

We now present our main theorem on the trade-off between personalization and global performance.

\begin{theorem}[Personalization-Performance Trade-off]
\label{theorem:personalization_performance_tradeoff}
Under Assumptions~\ref{assumption:L-smoothness}--\ref{assumption:bounded-human-feedback} and Definition~\ref{def:Rmax}--\ref{def:global_performance}, for any set of client policies \(\{\pi_k(\cdot|s,\theta)\}_{k=1}^K\) and the global policy \(\pi(\cdot|s,\theta)\), the global performance metric satisfies:
\[
J_g(\theta) \geq \frac{1}{K}\sum_{k=1}^K J_k^0(\pi_k) - C \cdot \left( \frac{1}{K}\sum_{k=1}^K \sqrt{P_k(\theta)} \right),
\]
where \( C > 0 \) is a constant given by:
$C = \frac{2 \sqrt{2} R_{\text{total}, \max}}{(1 - \gamma)^2},$
and \( R_{\text{total}, \max} = R_{\text{max}} + \lambda H_{\max} \) is the maximum possible total reward.
\end{theorem}

Theorem~\ref{theorem:personalization_performance_tradeoff} establishes that the global performance \( J_g(\theta) \) is lower bounded by the average client-specific performance (intrinsic rewards) \( \frac{1}{K}\sum_{k=1}^K J_k^0(\pi_k) \) minus a penalty term proportional to the average of the square roots of the personalization scores \( \frac{1}{K}\sum_{k=1}^K \sqrt{P_k(\theta)} \). The constant \( C \) encapsulates the maximum possible total reward and the discount factor, indicating that in environments with higher rewards or longer planning horizons, the impact of personalization on global performance is more pronounced.


\subsection{Impact of Human Feedback}

We analyze how the incorporation of human feedback, governed by the weight \( \lambda \), influences personalization and global performance.

\begin{theorem}[Impact of Human Feedback]
\label{theorem:impact_human_feedback}
Under the same assumptions and definitions in Theorem~\ref{theorem:personalization_performance_tradeoff}, as the human feedback weight \( \lambda \) increases:
\begin{enumerate}
    \item The average personalization score \( \frac{1}{K}\sum_{k=1}^K P_k(\theta) \) increases at a rate of \( O(\lambda^2) \).
    \item The global performance \( J_g(\theta) \) decreases at a rate of \( O(\lambda) \).
    \item The sample complexity \( N \) increases at a rate of \( O(\lambda) \).
\end{enumerate}
\end{theorem}

Theorem~\ref{theorem:impact_human_feedback} quantitatively demonstrates that increasing the human feedback weight $\lambda$ intensifies personalization (as the personalization score grows at $O(\lambda^2)$) but leads to a linear decrease in global performance and an increase in sample complexity.
\begin{enumerate}
    \item \textbf{Personalization Score Increases at \( O(\lambda^2) \):} The personalization score \( P_k(\theta) \) for each client scales quadratically with \( \lambda \), indicating that the degree of personalization becomes more pronounced as \( \lambda \) increases.
    \item \textbf{Global Performance Decreases at \( O(\lambda) \):} The global performance \( J_g(\theta) \) experiences a linear decrease with respect to \( \lambda \). This implies that while personalization enhances client-specific performance, it concurrently introduces a controlled degradation in overall system performance.
    \item \textbf{Sample Complexity Increases at \( O(\lambda) \):} The total number of samples \( N \) required to achieve a desired level of performance grows linearly with \( \lambda \). This reflects the increased data demands associated with higher levels of personalization to maintain convergence guarantees.
\end{enumerate}

\section{Empirical Results}\label{sec:experimental_results}
We evaluate FedRLHF's effectiveness in integrating human feedback within a federated reinforcement learning setting through two real-world tasks: movie rating prediction using the MovieLens dataset and sentiment-controlled review generation using the IMDb dataset. Our experiments benchmark FedRLHF against a centralized RLHF baseline, with a focus on personalization, and maintaining performance levels.

All experiments were conducted on an NVIDIA GeForce RTX 3090 GPU, using the Flower framework~\cite{beutel2020flower} to simulate a realistic federated learning environment with gRPC communication, mimicking real-world distributed systems. Detailed experimental results and analyses are provided in Appendix~\ref{appendix:sec:full-experiments-details}. 

\subsection{Movie Rating Prediction on MovieLens}
\label{sec:movielens_experiment}

\subsubsection{Task Description and Setup}
In this task, we simulate a streaming service enhancing its recommendation system while preserving user privacy and catering to individual preferences. Using the \texttt{ml-latest-small} version of the MovieLens dataset~\cite{harper2015movielens} which contains 100,836 ratings from 610 users on 9,742 movies, we randomly selected $K=10$ users as clients, each with unique viewing histories and preferences. The objective is to predict whether a user would assign a high rating (4 stars or above) to a given movie, effectively framing this as a binary classification task.

\subsubsection{Human Feedback Simulation}\label{subsec:hf-simulation}
To emulate realistic user behavior and feedback mechanisms, we developed a noise-aware, rule-based feedback simulator generating two types of feedback:
 \textbf{a. Direct Feedback:} Categorizes predictions as "too high" (-1), "too low" (1), or "about right" (0) based on the difference between predicted and actual ratings;
 \textbf{b. Comparative Feedback:} Expresses preferences between movie pairs, mirroring real-world scenarios where users more easily compare options than provide absolute ratings.
Feedback values are bounded within [-1, 1], satisfying Assumption~\ref{assumption:bounded-human-feedback}. This feedback trains a local reward (preference) model for each client.
Full details on the feedback simulation are provided in Appendix~\ref{appendix:subsec:human-feedback-details}.

\subsubsection{Implementation}
We implemented a neural network model with embedding layers for users and movies.
The model inputs included user IDs, movie IDs, and movie genre information to capture complex user-movie interactions.
In the federated learning process, each client trained the model locally using intrinsic rewards and simulated human feedback, employing Q-learning as the local RLHF step. 
We employ Q-learning in this task due to its effectiveness in handling discrete action spaces (movie recommendations) and its ability to learn optimal policies in environments with delayed rewards.
Clients performed 5 local epochs per federated round, using the Adam optimizer \cite{kingma2014adam} (learning rate $1 \times 10^{-3}$). 
Global model aggregation used a weighted average based on client example counts, following a FedAvg variant~\cite{mcmahan2017communication}. The process spanned 5 communication rounds. More details are provided in Appendix \ref{appendix:subsec:movielens-implementation-details}.

\begin{figure}[t]
    \centering
    \includegraphics[width=0.7\linewidth]
    {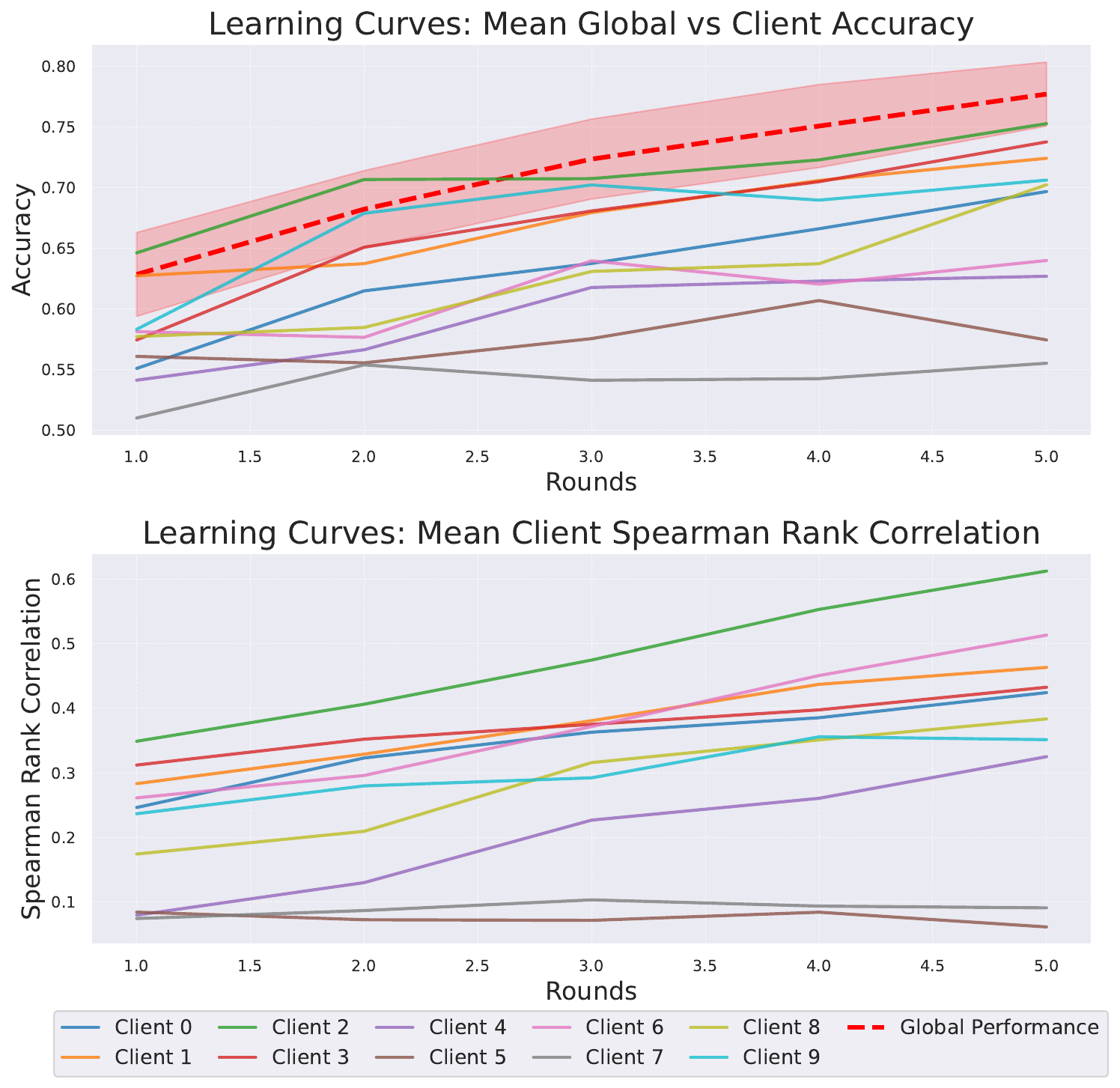}
    \caption{Learning curves on MovieLens: (top) Global vs. Client Accuracy, (bottom) Client Spearman correlation.}
    \label{fig:movielens-performance-k10}
        \Description{A plot showing the learning curves of client accuracy and Spearman correlation for the MovieLens dataset. The global accuracy improves steadily across rounds.}
\end{figure}

\subsubsection{Results and Analysis}
Figure \ref{fig:movielens-performance-k10} presents the learning curves for both global and client-specific accuracies (top), along with the Spearman rank correlations (bottom) for each of the $K=10$ clients across 5 federated rounds. All results are averaged over five independent runs. 
For clarity, only client means across runs are shown. Global performance is depicted by the red dashed line (mean) with shaded area (standard deviation).

\paragraph{Global Performance Improvements} The global performance in accuracy improves from $62.86\% \pm 3.45\%$ to $77.71\% \pm 2.64\%$ over 5 rounds.
The steady improvement in global performance, which is also evident from the client-specific accuracies distribution shown in the violin plot in Figure~\ref{fig:movielens-distribution-k10} (top subplot),
aligns with the $O(1/T)$ convergence rate established in Theorem~\ref{theorem:convergence}.

\paragraph{Personalization-Performance Trade-off}
We use Spearman rank correlation to evaluate how well the model captured user-specific movie preferences, serving as a surrogate for the personalization-performance trade-off discussed in Theorem~\ref{theorem:personalization_performance_tradeoff}. 
Figure~\ref{fig:movielens-performance-k10} (bottom) reveals substantial variability in Spearman correlations across clients, ranging from $0.0613$ (Client 5) to $0.6126$ (Client 2) in the last round, with high correlations indicating effective personalization and low correlations suggesting challenges in capturing nuanced preferences.
The upward trend in median Spearman correlations (Figure~\ref{fig:movielens-distribution-k10} bottom) demonstrates the framework's increasing ability to develop personalized models aligned with individual preferences, while steadily improving global performance (Figure~\ref{fig:movielens-distribution-k10} top).


\paragraph{Scaling to $K=50$ Clients}
Similar trends were observed when scaling to $K=50$ clients, with details provided in 
Appendix~\ref{appendix:para:movielens-results-k50}.

\begin{figure}[t]
    \centering
    \includegraphics[width=0.7\linewidth]{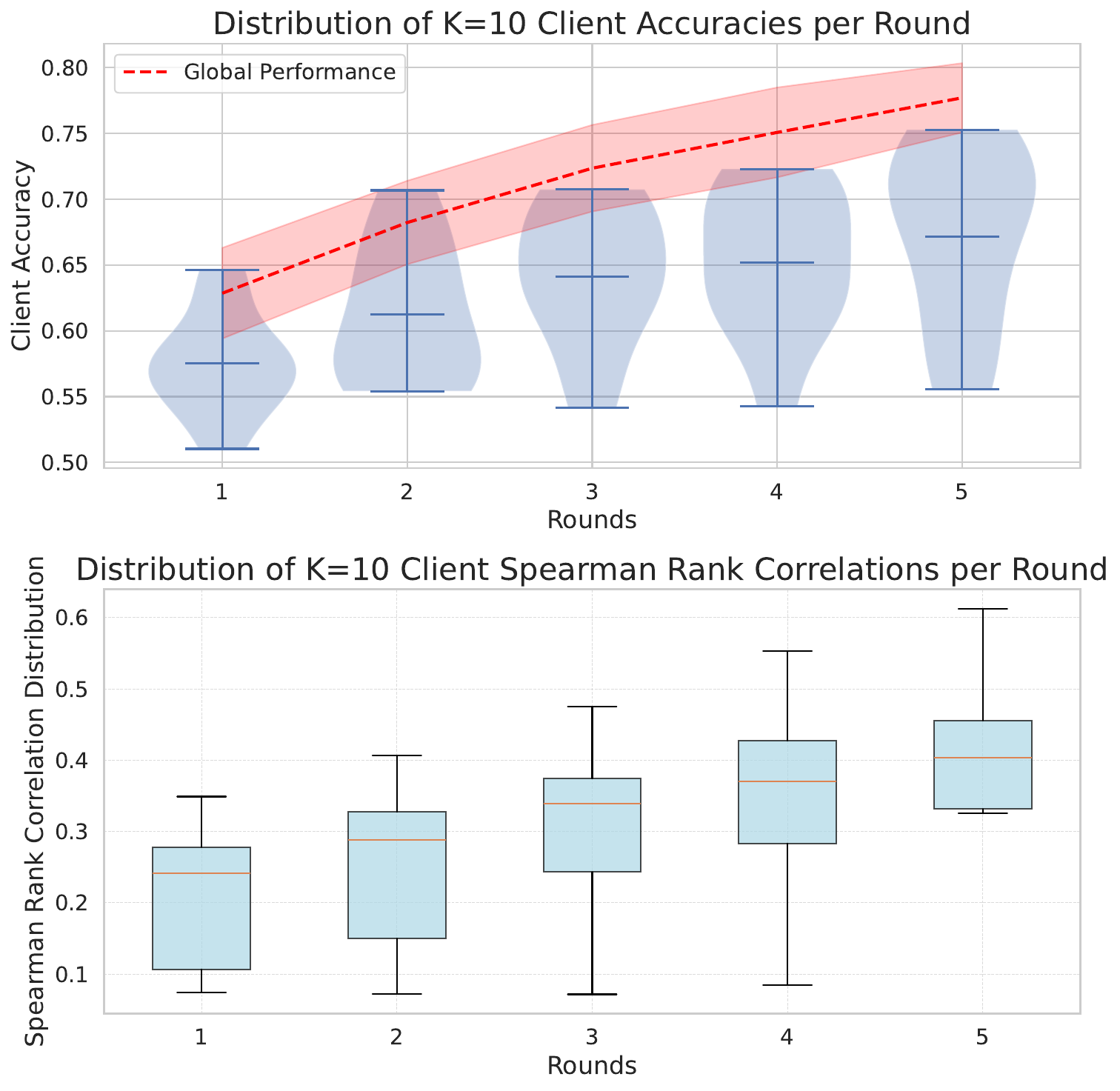}
    \caption{Distribution of $K=10$ client accuracies and Spearman rank correlations per round for the MovieLens task.}
    \label{fig:movielens-distribution-k10}
    \Description{Violin and box plots showing the distribution of client accuracies and Spearman rank correlations for the MovieLens task.}
\end{figure}

\subsection{Sentiment-Controlled Review Generation}

\subsubsection{Task Description and Setup}
In this task
we simulate multiple movie review platforms collaborating to fine-tune a language model for sentiment-controlled text generation without sharing data. 
Each client represents a distinct platform with its own collection of movie reviews, introducing natural data heterogeneity.
Using the IMDb dataset~\cite{maas2011learning}, we partitioned 50,000 reviews among $K=5$ clients, each receiving approximately 10,000 unique reviews.

\subsubsection{Implementation (details in Appendix \ref{appendix:subsec:IMDb_implementation_details})}
We employed a GPT-2 model~\cite{radford2019language} fine-tuned using PPO~\cite{schulman2017proximal-PPO} within the TRL library~\cite{vonwerra2022trl}. Clients conducted local RLHF training for 5 epochs per federated round, using Adam optimizer (learning rate $1\times10^{-5}$). Global aggregation used FedAvg~\cite{mcmahan2017communication} over 5 communication rounds.

\subsubsection{Human Feedback Simulation}
We simulated human feedback using a sentiment analysis model (DistilBERT~\cite{sanh2019distilbert} fine-tuned on IMDb) implemented locally on each client. The reward function combined sentiment score and language model log probability:
$R_k = \lambda_k \cdot R_{\text{sentiment}} + (1-\lambda_k) \cdot R^0_k
$
, where $R_{\text{sentiment}}$ is the sentiment alignment reward, $R^0_k$ is the intrinsic fluency reward and $\lambda_k \in [0,1]$ is a client-specific parameter controlling personalization. This formulation closely aligns with the reward shaping
in Equation~\ref{eq:reward-shaping} and validates Assumption~\ref{assumption:bounded-human-feedback}, allowing clients to personalize the importance of sentiment alignment.

\subsubsection{Results and Analysis} We report the results using a single random seed (42) to maintain consistency across experiments.


\begin{figure}[ht]
    \centering
    \begin{subfigure}[t]{0.99\linewidth}
        \centering
        \includegraphics[width=0.7\linewidth,height=0.73\linewidth]{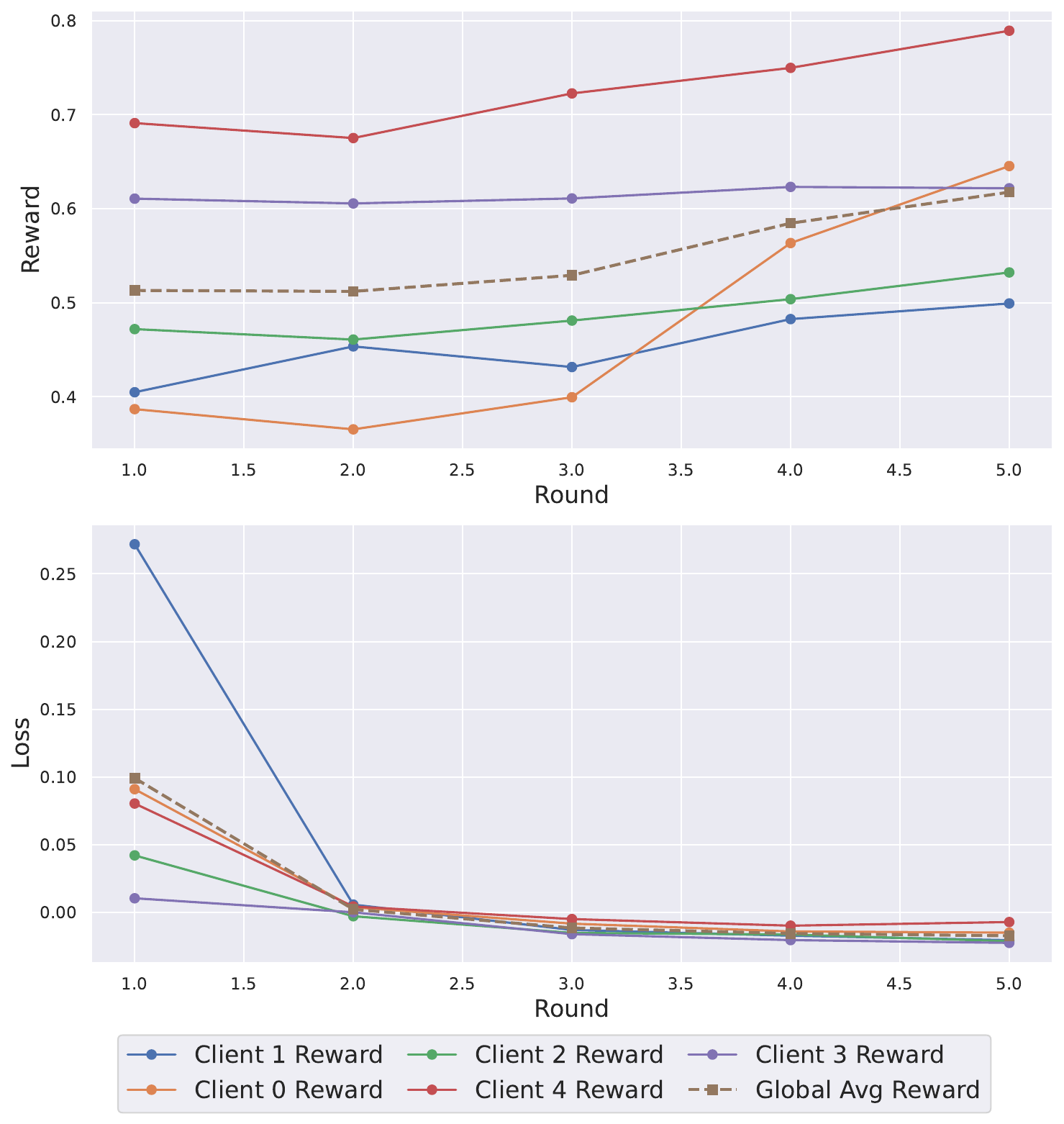}  
        \caption{Global and clients performance of FedRLHF in the IMDb task.}
        \label{fig:imdb_global_performance}
    \end{subfigure}
    \hspace{0.01\linewidth}
    \begin{subfigure}[t]{0.99\linewidth}
        \centering
                \includegraphics[width=0.7\linewidth]{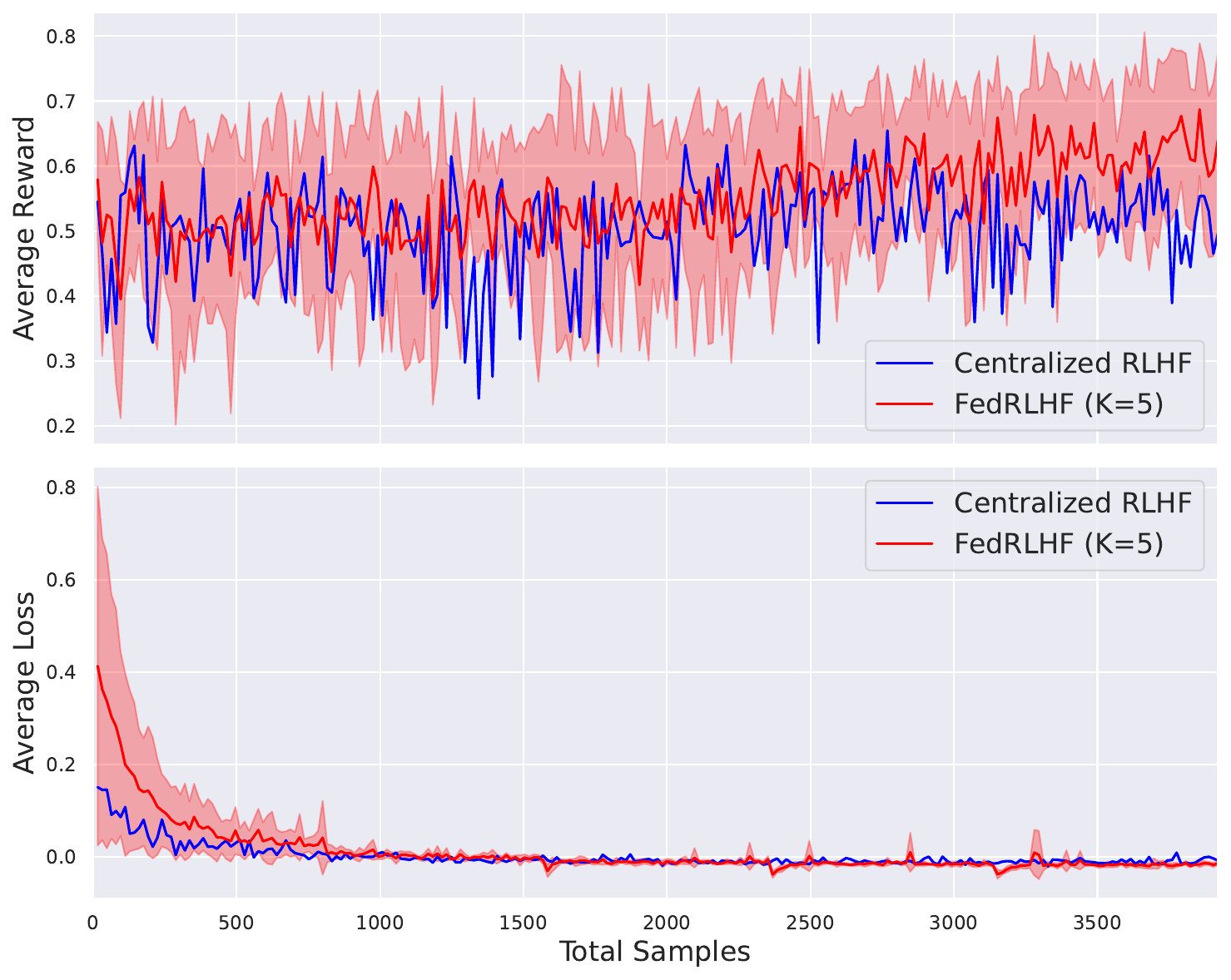}  
        \caption{Sample efficiency of FedRLHF ($K=5$) versus Centralized RLHF.}
        \label{fig:imdb_sample_efficiency_compare}
    \end{subfigure}
    \caption{Performance evaluation of FedRLHF in comparison to centralized RLHF. (a) Tracks the rewards and losses of clients and the global performance over federation rounds. 
    (b) Compares the sample efficiency of FedRLHF (K = 5) with centralized RLHF in terms of average rewards and losses over training samples.}
    \label{fig:combined}
\end{figure}

\paragraph{Comparison with Centralized RLHF}
Figure~\ref{fig:imdb_sample_efficiency_compare} compares the average rewards and losses between centralized RLHF and FedRLHF ($K=5$) over the total number of samples. The rewards comparison reveals that while the centralized model initially achieves slightly higher rewards, FedRLHF quickly catches up and even surpasses the centralized model's performance in later stages. 
This is evident from the FedRLHF reward curve (in red) consistently lying above the centralized RLHF curve (in blue) after approximately 1500 samples.
This
improvement arises from FedRLHF’s ability to leverage diverse client data and regular model aggregation, which enhance generalization and
reduce overfitting compared to the centralized approach. 
The loss comparison shows that both approaches achieve similar loss reduction. This result corroborates the sample complexity analysis in Theorem~\ref{theorem:sample-complexity}, indicating that FedRLHF can match or even exceed centralized performance while preserving privacy and distributing computation.

\paragraph{Personalization-Performance Trade-off}
The global average reward, represented by the dashed line in Figure~\ref{fig:imdb_global_performance}, shows steady improvement from approximately 0.52 to 0.68 over five rounds, indicating overall system convergence.
To analyze the trade-off between personalization and global performance in FedRLHF, we conducted a detailed evaluation of how clients' personalized objectives affected their individual rewards over the training rounds. For each client, we randomly sampled 30 queries from their evaluation dataset at the beginning of training and kept these queries fixed throughout all rounds. 
Each client was assigned a different personalization weight $\lambda_k$, ranging from 0.1 to 0.9. 
In each communication round, we supplied these 30 queries to the client's GPT model, recorded the generated responses, and calculated the corresponding average of intrinsic rewards ($R_{\text{intrinsic}}$), sentiment rewards ($R_{\text{sentiment}}$), and combined rewards ($R_k$), as shown in Figure~\ref{fig:visualization_sentiment_int_rewards}.

The results reveal distinct patterns aligned with the personalization weights $\lambda_k$. Client 0 ($\lambda_0 = 0.1$) prioritizes intrinsic rewards, while Client 1 ($\lambda_1 = 0.3$) shows more balanced improvement. Client 2 ($\lambda_2 = 0.5$) exhibits clear equilibrium between sentiment and intrinsic rewards. For Client 3 ($\lambda_3 = 0.7$), sentiment rewards dominate with a steady increase, and Client 4 ($\lambda_4 = 0.9$) demonstrates the highest emphasis on sentiment rewards. As $\lambda$ increases across clients, we observe a clear shift from intrinsic to sentiment reward prioritization, with combined rewards increasingly aligning with sentiment rewards for higher $\lambda$ values.



\begin{figure}[t]
    \centering
    \includegraphics[width=0.8\linewidth, height=0.7\linewidth]{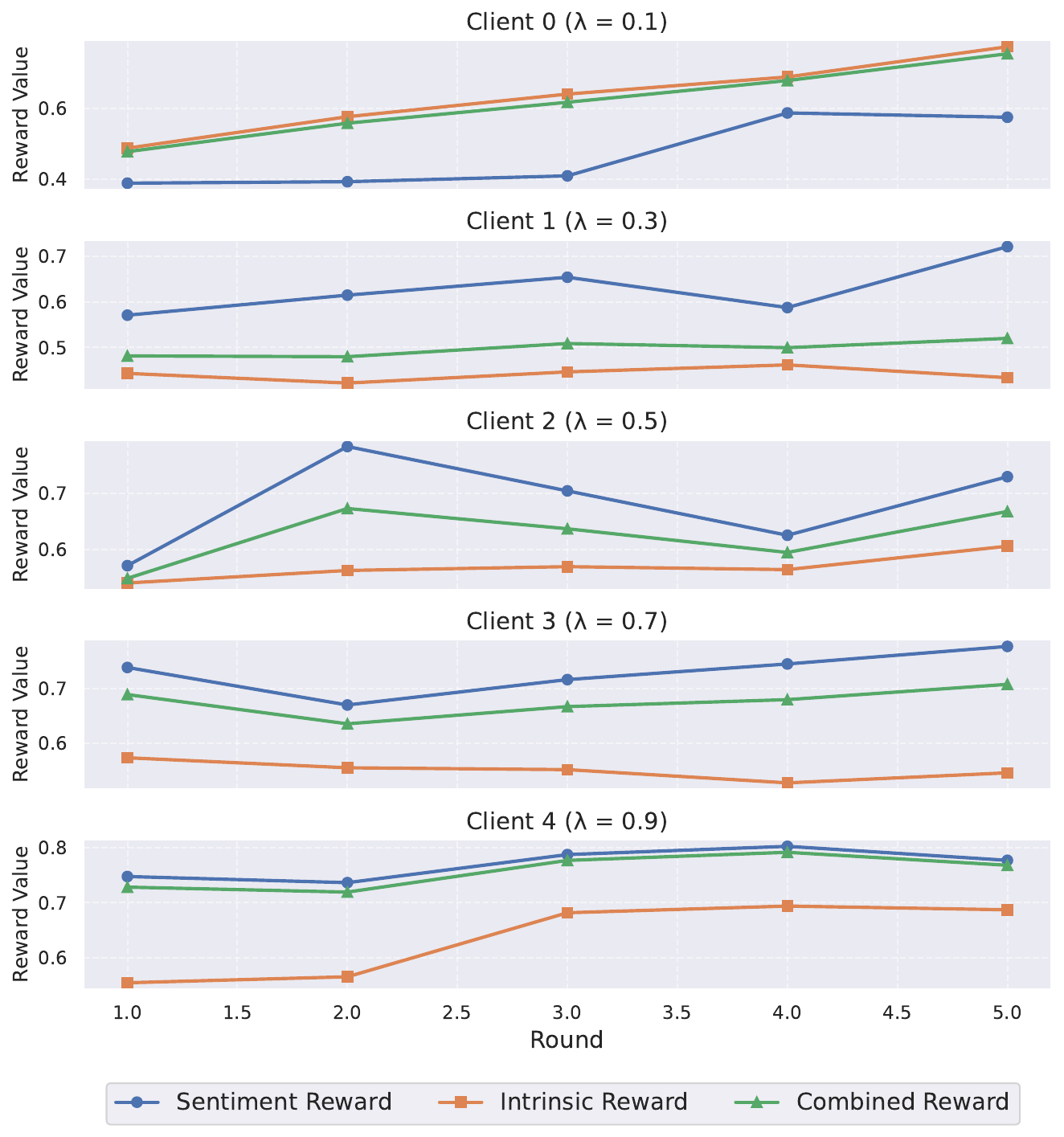}
    \caption{Trends of intrinsic rewards, sentiment rewards, and combined rewards over communication rounds for each client. Each subplot corresponds to one client, illustrating personalization effects due to varying $\lambda_k$ values.}
        \Description{Trends of intrinsic rewards, sentiment rewards, and combined rewards over communication rounds for each client in the IMDb task. Each subplot corresponds to one client, illustrating personalization effects due to varying $\lambda_k$ values.}
    \label{fig:visualization_sentiment_int_rewards}
\end{figure}

\section{Conclusion and Future Work}\label{sec:conclusion}

In this work, we have introduced FedRLHF, a novel framework that integrates federated reinforcement learning principles with RLHF to address privacy and personalization challenges of data centralization.
Our theoretical analysis provides convergence guarantees and sample complexity bounds, demonstrating stable, linear convergence. The personalization-performance trade-off analysis shows how FedRLHF balances global performance with individual client needs. 
Empirical evaluations on MovieLens and IMDb validate our approach, achieving results comparable to centralized RLHF while preserving privacy and enhancing personalization.

Future work will focus on enhancing FedRLHF's robustness through advanced aggregation techniques and strengthening privacy preservation by integration of formal privacy guarantees, such as differential privacy.
Additionally, we aim to investigate the trade-off between communication efficiency and personalization, optimizing FedRLHF's performance by balancing communication overhead with personalized model adaptations in federated environments.



\bibliographystyle{ACM-Reference-Format} 
\bibliography{sample}

\newpage
\appendix
\onecolumn

\section{Proofs of useful lemmas}\label{appendix:sec:proofs_useful_lemmas}
\renewcommand{\thesection}{\Alph{section}}
\setcounter{lemma}{0}
\renewcommand{\thelemma}{\thesection.\arabic{lemma}}

\begin{lemma}[Restatement of Lemma~~\ref{lemma:bounded-local-global-difference} Bounded Local-Global Difference]
Under Assumptions~\ref{assumption:L-smoothness}, \ref{assumption:G-bounded-gradients}, and \ref{assumption:bounded-variance}, 
for any communication round $t$ and client $k$, we have:
\[
\mathbb{E}\left[\|\theta_t^k - \theta_t\|^2\right] \leq \eta^2\tau^2(G^2 + \sigma^2)
\]
where $\theta_t^k$ is the local model of client $k$, $\theta_t$ is the global model, $\eta$ is the learning rate, $\tau$ is the number of local updates, $G$ is the gradient bound, and $\sigma^2$ is the variance bound.
\end{lemma}

\begin{proof}
We express the difference between the local and global models after $\tau$ local updates as:
\[
\theta_t^k - \theta_t = \sum_{i=0}^{\tau-1} \eta\nabla\hat{J}_k(\theta_{t,i}^k; B_{t,i}^k)
\]

Taking the squared norm and expectation:
\begin{align*}
\mathbb{E}\left[\|\theta_t^k - \theta_t\|^2\right] &= \mathbb{E}\left[\left\|\sum_{i=0}^{\tau-1} \eta\nabla\hat{J}_k(\theta_{t,i}^k; B_{t,i}^k)\right\|^2\right] \\
&= \eta^2\mathbb{E}\left[\left\|\sum_{i=0}^{\tau-1} \nabla\hat{J}_k(\theta_{t,i}^k; B_{t,i}^k)\right\|^2\right]
\end{align*}

Using the inequality $(a_1 + ... + a_n)^2 \leq n(a_1^2 + ... + a_n^2)$, which follows from Jensen's inequality:
\begin{align*}
\mathbb{E}\left[\|\theta_t^k - \theta_t\|^2\right] &\leq \eta^2\tau\sum_{i=0}^{\tau-1} \mathbb{E}\left[\|\nabla\hat{J}_k(\theta_{t,i}^k; B_{t,i}^k)\|^2\right]
\end{align*}

Now, we can use Assumptions~\ref{assumption:G-bounded-gradients} (G-bounded gradients) and~\ref{assumption:bounded-variance} ($\sigma$-bounded variance):
\begin{align*}
\mathbb{E}\left[\|\nabla\hat{J}_k(\theta_{t,i}^k; B_{t,i}^k)\|^2\right] &= \mathbb{E}\left[\|\nabla J_k(\theta_{t,i}^k) + (\nabla\hat{J}_k(\theta_{t,i}^k; B_{t,i}^k) - \nabla J_k(\theta_{t,i}^k))\|^2\right] \\
&\leq \mathbb{E}\left[\|\nabla J_k(\theta_{t,i}^k)\|^2\right] + \mathbb{E}\left[\|\nabla\hat{J}_k(\theta_{t,i}^k; B_{t,i}^k) - \nabla J_k(\theta_{t,i}^k)\|^2\right] \\
&\leq G^2 + \sigma^2
\end{align*}

Substituting this back:
\begin{align*}
\mathbb{E}\left[\|\theta_t^k - \theta_t\|^2\right] &\leq \eta^2\tau\sum_{i=0}^{\tau-1} (G^2 + \sigma^2) \\
&= \eta^2\tau^2(G^2 + \sigma^2)
\end{align*}

This completes the proof.
\end{proof}

\begin{lemma}[Restatement of Lemma~\ref{lemma:one-step-descent}]
Under Assumptions~\ref{assumption:L-smoothness}--\ref{assumption:bounded-human-feedback}, for any round $t$, the expected improvement in the global objective satisfies:
\begin{align*}
\mathbb{E}[J(\theta_{t+1})] &\geq J(\theta_t) + \eta\tau\left(1 - \frac{L\eta\tau}{2}\right)\|\nabla J(\theta_t)\|^2 - \frac{L}{2}\left(\frac{\eta^2\tau^2}{K}\right)(G^2 + \sigma^2) - \lambda H_{\max}
\end{align*}
where $\theta_{t+1}$ is the updated global model.
\end{lemma}

\begin{proof}
From the $L$-smoothness of $J$ (since it's an average of $L$-smooth functions), we have:
\begin{align}
J(\theta_{t+1}) \geq J(\theta_t) + \langle \nabla J(\theta_t), \theta_{t+1} - \theta_t \rangle - \frac{L}{2} \|\theta_{t+1} - \theta_t\|^2.
\label{eq:L-smooth}
\end{align}

The global model update is:
\[
\theta_{t+1} = \theta_t + \frac{1}{K} \sum_{k=1}^K (\theta_t^k - \theta_t).
\]

Substituting $\theta_{t+1} - \theta_t$ into \eqref{eq:L-smooth}:
\begin{align*}
J(\theta_{t+1}) &\geq J(\theta_t) + \left\langle \nabla J(\theta_t), \frac{1}{K} \sum_{k=1}^K (\theta_t^k - \theta_t) \right\rangle - \frac{L}{2} \left\| \frac{1}{K} \sum_{k=1}^K (\theta_t^k - \theta_t) \right\|^2.
\end{align*}

Taking expectation over the stochasticity, we have:
\begin{align}
\mathbb{E}\left[ J(\theta_{t+1}) \right] &\geq J(\theta_t) + \frac{1}{K} \sum_{k=1}^K \mathbb{E}\left[ \left\langle \nabla J(\theta_t), \theta_t^k - \theta_t \right\rangle \right] - \frac{L}{2} \mathbb{E}\left[ \left\| \frac{1}{K} \sum_{k=1}^K (\theta_t^k - \theta_t) \right\|^2 \right].
\label{eq:expected-improvement}
\end{align}

\textbf{Bounding the First Term:}

We focus on the inner product term:
\[
\mathbb{E}\left[ \left\langle \nabla J(\theta_t), \theta_t^k - \theta_t \right\rangle \right] = \mathbb{E}\left[ \left\langle \nabla J(\theta_t), \sum_{i=0}^{\tau-1} \eta \nabla \hat{J}_k(\theta_{t,i}^k) \right\rangle \right].
\]

We can decompose the stochastic gradient into:
\[
\nabla \hat{J}_k(\theta_{t,i}^k) = \nabla J_k(\theta_{t,i}^k) + \delta_{t,i}^k,
\]
where $\delta_{t,i}^k = \nabla \hat{J}_k(\theta_{t,i}^k) - \nabla J_k(\theta_{t,i}^k)$ is the stochastic gradient noise with zero mean and variance bounded by $\sigma^2$.

Therefore, we have:
\begin{align*}
\mathbb{E}\left[ \left\langle \nabla J(\theta_t), \theta_t^k - \theta_t \right\rangle \right] &= \mathbb{E}\left[ \sum_{i=0}^{\tau-1} \eta \left\langle \nabla J(\theta_t), \nabla J_k(\theta_{t,i}^k) + \delta_{t,i}^k \right\rangle \right] \\
&= \eta \sum_{i=0}^{\tau-1} \mathbb{E}\left[ \left\langle \nabla J(\theta_t), \nabla J_k(\theta_{t,i}^k) \right\rangle \right],
\end{align*}
since $\mathbb{E}[\delta_{t,i}^k] = 0$.

We can further decompose $\nabla J_k(\theta_{t,i}^k)$:
\[
\nabla J_k(\theta_{t,i}^k) = \nabla J_k(\theta_t) + \left( \nabla J_k(\theta_{t,i}^k) - \nabla J_k(\theta_t) \right).
\]

By Lipschitz continuity (Assumption~\ref{assumption:L-smoothness}):
\[
\left\| \nabla J_k(\theta_{t,i}^k) - \nabla J_k(\theta_t) \right\| \leq L \left\| \theta_{t,i}^k - \theta_t \right\|.
\]

Therefore, the inner product can be bounded using Cauchy-Schwarz:
\begin{align*}
\left\langle \nabla J(\theta_t), \nabla J_k(\theta_{t,i}^k) \right\rangle &= \left\langle \nabla J(\theta_t), \nabla J_k(\theta_t) \right\rangle + \left\langle \nabla J(\theta_t), \nabla J_k(\theta_{t,i}^k) - \nabla J_k(\theta_t) \right\rangle \\
&\geq \left\| \nabla J(\theta_t) \right\|^2 - \left\| \nabla J(\theta_t) \right\| L \left\| \theta_{t,i}^k - \theta_t \right\|.
\end{align*}

Taking expectation:
\begin{align*}
\mathbb{E}\left[ \left\langle \nabla J(\theta_t), \nabla J_k(\theta_{t,i}^k) \right\rangle \right] &\geq \left\| \nabla J(\theta_t) \right\|^2 - L \left\| \nabla J(\theta_t) \right\| \mathbb{E}\left[ \left\| \theta_{t,i}^k - \theta_t \right\| \right].
\end{align*}

Using Lemma~\ref{lemma:bounded-local-global-difference}, we have:
\[
\mathbb{E}\left[ \left\| \theta_{t,i}^k - \theta_t \right\| \right] \leq \sqrt{ \mathbb{E}\left[ \left\| \theta_{t,i}^k - \theta_t \right\|^2 \right] } \leq 
\eta \tau  \sqrt{G^2 + \sigma^2} .
\]

Substituting back:
\begin{align*}
\mathbb{E}\left[ \left\langle \nabla J(\theta_t), \nabla J_k(\theta_{t,i}^k) \right\rangle \right] &\geq \left\| \nabla J(\theta_t) \right\|^2 - L \left\| \nabla J(\theta_t) \right\| \eta \tau 
\sqrt{G^2 + \sigma^2}.
\end{align*}

Summing over $i$ and averaging over $k$:
\begin{align*}
\frac{1}{K} \sum_{k=1}^K \mathbb{E}\left[ \left\langle \nabla J(\theta_t), \theta_t^k - \theta_t \right\rangle \right] &\geq \eta \tau \left\| \nabla J(\theta_t) \right\|^2 - \eta^2 \tau^2 L \left\| \nabla J(\theta_t) \right\| 
\sqrt{G^2 + \sigma^2}
\end{align*}

For small $\eta$, the second term is of higher order and can be neglected in the first-order analysis, or bounded appropriately.

\textbf{Bounding the Second Term:}

For the squared norm term, we have:
\begin{align*}
\mathbb{E}\left[ \left\| \frac{1}{K} \sum_{k=1}^K (\theta_t^k - \theta_t) \right\|^2 \right] &= \frac{1}{K^2} \mathbb{E}\left[ \left\| \sum_{k=1}^K (\theta_t^k - \theta_t) \right\|^2 \right] \\
&= \frac{1}{K^2} \sum_{k=1}^K \mathbb{E}\left[ \left\| \theta_t^k - \theta_t \right\|^2 \right] + \frac{1}{K^2} \sum_{k \neq k'} \mathbb{E}\left[ \left\langle \theta_t^k - \theta_t, \theta_t^{k'} - \theta_t \right\rangle \right].
\end{align*}

Assuming independence across clients and zero-mean stochastic gradients, the cross terms vanish, and we have:
\[
\mathbb{E}\left[ \left\| \frac{1}{K} \sum_{k=1}^K (\theta_t^k - \theta_t) \right\|^2 \right] \leq \frac{1}{K} \mathbb{E}\left[ \left\| \theta_t^k - \theta_t \right\|^2 \right].
\]

Using Lemma~\ref{lemma:bounded-local-global-difference}:
\[
\mathbb{E}\left[ \left\| \frac{1}{K} \sum_{k=1}^K (\theta_t^k - \theta_t) \right\|^2 \right] \leq \frac{\eta^2 \tau^2}{K} \left( G^2 + \sigma^2 \right).
\]

\textbf{Combining Both Terms:}

Substituting the bounds back into \eqref{eq:expected-improvement}, we have:
\begin{align*}
\mathbb{E}\left[ J(\theta_{t+1}) \right] \geq J(\theta_t) + \eta \tau \left\| \nabla J(\theta_t) \right\|^2 - \eta^2 \tau^2 L \left\| \nabla J(\theta_t) \right\| 
\sqrt{G^2 + \sigma ^2}
- \frac{L}{2} \left( \frac{\eta^2 \tau^2}{K} 
\left( G^2 + \sigma^2 \right) \right).
\end{align*}

\begin{enumerate}
\item First, let's focus on the term $\eta^2\tau^2L\|\nabla J(\theta_t)\| \sqrt{G^2 + \sigma^2}$. We can bound this using the Cauchy-Schwarz inequality:

\begin{align*}
\eta^2\tau^2L\|\nabla J(\theta_t)\| \sqrt{G^2 + \sigma^2} &\leq \frac{\eta\tau L}{2}\|\nabla J(\theta_t)\|^2 + \frac{\eta^3\tau^3L}{2}(G^2 + \sigma^2)
\end{align*}

\item Substituting this bound back into our original inequality:

\begin{align*}
\mathbb{E}[J(\theta_{t+1})] &\geq J(\theta_t) + \eta\tau \|\nabla J(\theta_t)\|^2 - \frac{\eta\tau L}{2}\|\nabla J(\theta_t)\|^2 - \frac{\eta^3\tau^3L}{2}(G^2 + \sigma^2) - \frac{L}{2}\frac{\eta^2\tau^2}{K}(G^2 + \sigma^2) \\
&= J(\theta_t) + \eta\tau\left(1 - \frac{L\eta}{2}\right)\|\nabla J(\theta_t)\|^2 - \frac{L}{2}\left(\frac{\eta^2\tau^2}{K} + \eta^3\tau^3\right)(G^2 + \sigma^2)
\end{align*}

\item Assuming $\eta$ is small and neglecting higher-order terms, we simplify:

\begin{align*}
\mathbb{E}[J(\theta_{t+1})] &\geq  J(\theta_t) + \eta\tau\left(1 - \frac{L\eta}{2}\right)\|\nabla J(\theta_t)\|^2 - \frac{L}{2}\left(\frac{\eta^2\tau^2}{K} + \eta^3\tau^3\right)(G^2 + \sigma^2) \\
&\geq J(\theta_t) + \eta\tau\left(1 - \frac{L\eta}{2}\right)\|\nabla J(\theta_t)\|^2 - \frac{L}{2}\left(\frac{\eta^2\tau^2}{K} \right)(G^2 + \sigma^2)
\end{align*}

\end{enumerate}

\textbf{Accounting for human feedback:}
Since human feedback \( H_k(s, a) \) is bounded by \( H_{\max} \) (as per Assumption~\ref{assumption:bounded-human-feedback}), the aggregated impact of human feedback across all clients introduces a worst-case adjustment of \( \lambda H_{\max} \) to the global objective.
We subtract this term from our lower bound:
\begin{align*}
\mathbb{E}[J(\theta_{t+1})] &\geq J(\theta_t) + \eta\tau\left(1 - \frac{L\eta\tau}{2}\right)\|\nabla J(\theta_t)\|^2 - \frac{L}{2}\left(\frac{\eta^2\tau^2}{K}\right)(G^2 + \sigma^2) - \lambda H_{\max}
\end{align*}
 This ensures that while human feedback guides the policy towards user preferences, its influence remains controlled and does not detrimentally affect the convergence guarantees.
\end{proof}

\section{Proofs of main theorems for convergence and sample complexity}\label{appendix:sec:proofs_theorems_convergence}
\renewcommand{\thesection}{\Alph{section}}
\setcounter{theorem}{0}
\renewcommand{\thetheorem}{\thesection.\arabic{theorem}}

\begin{theorem}[Restatement of Theorem~\ref{theorem:convergence} Convergence of FedRLHF]
Under Assumptions~\ref{assumption:L-smoothness}--\ref{assumption:bounded-human-feedback}, if we choose the constant learning rate $\eta = \frac{1}{L\tau}$,
then the output \(\theta_{\text{avg}} = \frac{1}{T} \sum_{t=0}^{T-1} \theta_t\) of Algorithm~\ref{alg:fedrlhf} satisfies:
\begin{align*}
\mathbb{E}\left[ J(\theta^*) - J(\theta_{\text{avg}}) \right] \leq \frac{L}{\mu T} \left( J(\theta^*) - J(\theta_0) \right)  + \frac{1}{2\mu K}(G^2 + \sigma^2) + \frac{L}{\mu}\lambda H_{\max}.
\end{align*}
\end{theorem}

\begin{proof}
We begin with the one-step descent inequality from Lemma~\ref{lemma:one-step-descent}:
\begin{align}
\mathbb{E}[J(\theta_{t+1})] &\geq J(\theta_t) + \eta\tau\left(1 - \frac{L\eta\tau}{2}\right)\|\nabla J(\theta_t)\|^2 - \frac{L}{2}\left(\frac{\eta^2\tau^2}{K}\right)(G^2 + \sigma^2) - \lambda H_{\max}.
\label{eq:one-step-descent}
\end{align}

Using the Polyak-Łojasiewicz (PL) condition (Assumption~\ref{assumption:PL-condition}):
\[
\|\nabla J(\theta_t)\|^2 \geq 2\mu \left( J(\theta^*) - J(\theta_t) \right) = 2\mu \Delta_t,
\]
where \(\Delta_t = J(\theta^*) - J(\theta_t)\).

Substituting this into \eqref{eq:one-step-descent}, we get:
\[
\mathbb{E}[J(\theta_{t+1})] \geq J(\theta_t) + 2\mu\eta\tau\left(1 - \frac{L\eta\tau}{2}\right)\Delta_t - \frac{L}{2}\left(\frac{\eta^2\tau^2}{K}\right)(G^2 + \sigma^2) - \lambda H_{\max}.
\]

Subtracting \( J(\theta^*) \) from both sides:
\[
\mathbb{E}[\Delta_{t+1}] \leq \left(1 - 2\mu\eta\tau\left(1 - \frac{L\eta\tau}{2}\right)\right)\Delta_t + \frac{L}{2}\left(\frac{\eta^2\tau^2}{K}\right)(G^2 + \sigma^2) + \lambda H_{\max},
\]
where we define the contraction factor:
\[
\rho = 1 - 2\mu\eta\tau\left(1 - \frac{L\eta\tau}{2}\right).
\]

We choose the constant learning rate:
\[
\eta = \frac{1}{L\tau}.
\]

Substituting \(\eta\) into \(\rho\):
\begin{align*}
\rho &= 1 - 2\mu\left(\frac{1}{L\tau}\right)\tau\left(1 - \frac{L\left(\frac{1}{L\tau}\right)\tau}{2}\right) \\
&= 1 - \frac{2\mu}{L}\left(1 - \frac{1}{2}\right) \\
&= 1 - \frac{\mu}{L}.
\end{align*}

We have the recursive inequality:
\[
\mathbb{E}[\Delta_{t+1}] \leq \rho \Delta_t + \epsilon_{\text{grad}} + \lambda H_{\max},
\]
where
\[
\epsilon_{\text{grad}} = \frac{L}{2}\left(\frac{\eta^2\tau^2}{K}\right)(G^2 + \sigma^2) = \frac{1}{2LK}(G^2 + \sigma^2).
\]

Unrolling this recursion:
\begin{align*}
\mathbb{E}[\Delta_{t}] &\leq \rho^{t} \Delta_0 + (\epsilon_{\text{grad}} + \lambda H_{\max}) \sum_{k=0}^{t-1} \rho^{k} \\
&= \rho^{t} \Delta_0 + (\epsilon_{\text{grad}} + \lambda H_{\max}) \frac{1 - \rho^{t}}{1 - \rho} \\
&\leq \rho^{t} \Delta_0 + \frac{\epsilon_{\text{grad}} + \lambda H_{\max}}{1 - \rho}.
\end{align*}

The average optimality gap over \( T \) iterations is:
\[
\Delta_{\text{avg}} = \frac{1}{T} \sum_{t=0}^{T-1} \mathbb{E}[\Delta_t] \leq \frac{1}{T} \sum_{t=0}^{T-1} \left( \rho^{t} \Delta_0 + \frac{\epsilon_{\text{grad}} + \lambda H_{\max}}{1 - \rho} \right ).
\]

Calculating the sum:
\begin{align*}
\sum_{t=0}^{T-1} \rho^{t} &= \frac{1 - \rho^{T}}{1 - \rho} \leq \frac{1}{1 - \rho}.
\end{align*}

Therefore,
\[
\Delta_{\text{avg}} \leq \frac{\Delta_0}{T(1 - \rho)} + \frac{\epsilon_{\text{grad}} + \lambda H_{\max}}{1 - \rho}.
\]

Recall that \( 1 - \rho = \frac{\mu}{L} \), so:
\[
\frac{1}{1 - \rho} = \frac{L}{\mu}.
\]

Substituting back:
\[
\Delta_{\text{avg}} \leq \frac{L}{\mu T} \Delta_0 + \frac{L}{\mu} (\epsilon_{\text{grad}} + \lambda H_{\max}).
\]

Substituting the value of \(\epsilon_{\text{grad}}\):
\begin{align*}
\Delta_{\text{avg}} &\leq \frac{L}{\mu T} \Delta_0 + \frac{L}{\mu} \left(\frac{1}{2LK}(G^2 + \sigma^2) + \lambda H_{\max}\right) \\
&= \frac{L}{\mu T} (J(\theta^*) - J(\theta_0)) + \frac{1}{2\mu K}(G^2 + \sigma^2) + \frac{L}{\mu}\lambda H_{\max}.
\end{align*}

This completes the proof.
\end{proof}

\begin{theorem}[Restatement of Theorem~\ref{theorem:sample-complexity} Sample Complexity of FedRLHF]
Under Assumptions~\ref{assumption:L-smoothness}--\ref{assumption:bounded-human-feedback}, to achieve an expected optimality gap of 
\[
\mathbb{E}\left[ J(\theta^*) - J(\theta_{\text{avg}}) \right] \leq \epsilon,
\]
the total number of samples required across all clients is:
\[
N = O\left( \frac{L (G^2 + \sigma^2)}{\mu^2 \epsilon^2} \right),
\]
subject to:
\[
    K \geq O\left( \frac{G^2 + \sigma^2}{\mu \epsilon} \right), \qquad     \lambda H_{\max} \leq O\left( \frac{\mu \epsilon}{L} \right).
\]
\end{theorem}

\begin{proof}
We begin with the convergence bound established in Theorem~\ref{theorem:convergence}:
\[
\mathbb{E}\left[ J(\theta^*) - J(\theta_{\text{avg}}) \right] \leq \underbrace{\frac{L}{\mu T} \left( J(\theta^*) - J(\theta_0) \right)}_{\epsilon_T} + \underbrace{\frac{1}{2\mu K}(G^2 + \sigma^2)}_{\epsilon_V} + \underbrace{\frac{L}{\mu}\lambda H_{\max}}_{\epsilon_H},
\]
where:
\begin{itemize}
    \item \( \epsilon_T \) represents the error due to a finite number of communication rounds \( T \),
    \item \( \epsilon_V \) accounts for the error stemming from gradient variance and bounded gradients,
    \item \( \epsilon_H \) captures the error introduced by human feedback.
\end{itemize}

Our objective is to ensure that the total expected optimality gap is within \( \epsilon \):
\[
\epsilon_T + \epsilon_V + \epsilon_H \leq \epsilon.
\]

\textbf{Step 1: Allocating the Error Budget}

To satisfy the above inequality, we allocate the error budget proportionally to each term. Specifically, we set:
\[
\epsilon_T = \frac{\epsilon}{3}, \quad \epsilon_V = \frac{\epsilon}{3}, \quad \epsilon_H = \frac{\epsilon}{3}.
\]

\textbf{Step 2: Deriving Requirements for \( T \), \( K \), and \( \lambda H_{\max} \)}

From the allocation, we derive the following conditions:

1. Communication Rounds \( T \):
   \[
   \frac{L}{\mu T} \left( J(\theta^*) - J(\theta_0) \right) \leq \frac{\epsilon}{3} \implies T \geq \frac{3L \left( J(\theta^*) - J(\theta_0) \right)}{\mu \epsilon}.
   \]

2. Number of Clients \( K \):
   \[
   \frac{1}{2\mu K}(G^2 + \sigma^2) \leq \frac{\epsilon}{3} \implies K \geq \frac{3(G^2 + \sigma^2)}{2\mu \epsilon}.
   \]

3. Human Feedback Term \( \lambda H_{\max} \):
   \[
   \frac{L}{\mu}\lambda H_{\max} \leq \frac{\epsilon}{3} \implies \lambda H_{\max} \leq \frac{\mu \epsilon}{3L}.
   \]
   This constraint ensures that the influence of human feedback remains controlled and does not detrimentally affect the convergence guarantees.

\textbf{Step 3: Calculating Total Number of Samples \( N \)}

The total number of samples across all clients is:
\[
N = K \times T.
\]
Substituting the lower bounds for \( K \) and \( T \) derived above:
\[
N \geq \left( \frac{3(G^2 + \sigma^2)}{2\mu \epsilon} \right) \times \left( \frac{3L \left( J(\theta^*) - J(\theta_0) \right)}{\mu \epsilon} \right) = \frac{9L \left( J(\theta^*) - J(\theta_0) \right) (G^2 + \sigma^2)}{2\mu^2 \epsilon^2}.
\]
Simplifying constants and focusing on the asymptotic behavior, we express the sample complexity as:
\[
N = O\left( \frac{L (G^2 + \sigma^2)}{\mu^2 \epsilon^2} \right).
\]




\end{proof}

\section{Proofs of main theorems for personalization}\label{appendix:sec:proofs_theorem_personalization}

\begin{theorem}[Restatement of Theorem~\ref{theorem:personalization_performance_tradeoff} Personalization-Performance Trade-off]
Under Assumptions~\ref{assumption:L-smoothness}--\ref{assumption:bounded-human-feedback} and Definition~\ref{def:Rmax}--\ref{def:global_performance}, for any set of client policies \(\{\pi_k(\cdot|s,\theta)\}_{k=1}^K\) and the global policy \(\pi(\cdot|s,\theta)\), the global performance metric satisfies:
\[
J_g(\theta) \geq \frac{1}{K}\sum_{k=1}^K J_k^0(\pi_k) - C \cdot \left( \frac{1}{K}\sum_{k=1}^K \sqrt{P_k(\theta)} \right),
\]
where \( C > 0 \) is a constant given by:
\[
C = \frac{2 \sqrt{2} R_{\text{total}, \max}}{(1 - \gamma)^2},
\]
and \( R_{\text{total}, \max} = R_{\text{max}} + \lambda H_{\max} \) is the maximum possible total reward.
\end{theorem}

\begin{proof}
We begin by leveraging the Performance Difference Lemma (Lemma~\ref{lemma:performance_difference}), which relates the difference in value functions of two policies to their differences in action probabilities.

\begin{lemma}[Performance Difference Lemma]
\label{lemma:performance_difference}
For any policies \(\pi\) and \(\pi'\), the difference in expected cumulative rewards for client \( k \) is:
\[
J_k^0(\pi') - J_k^0(\pi) = \frac{1}{1 - \gamma} \mathbb{E}_{s \sim d^{\pi'}_k} \left[ \mathbb{E}_{a \sim \pi'} \left[ A_k^{\pi}(s, a) \right] \right],
\]
where \( d^{\pi'}_k(s) \) is the discounted state visitation distribution under policy \(\pi'\) for client \( k \), and \( A_k^{\pi}(s, a) \) is the advantage function of policy \(\pi\) for client \( k \).
\end{lemma}

Applying the lemma with \(\pi' = \pi_k\) and \(\pi = \pi\), we have:
\[
J_k^0(\pi_k) - J_k^0(\pi) = \frac{1}{1 - \gamma} \mathbb{E}_{s \sim d^{\pi_k}_k} \left[ \mathbb{E}_{a \sim \pi_k(\cdot|s,\theta)} \left[ A_k^{\pi}(s, a) \right] \right].
\]

Rewriting, we obtain:
\[
J_k^0(\pi) = J_k^0(\pi_k) - \frac{1}{1 - \gamma} \mathbb{E}_{s \sim d^{\pi_k}_k} \left[ \mathbb{E}_{a \sim \pi_k(\cdot|s,\theta)} \left[ A_k^{\pi}(s, a) \right] \right].
\]

Our goal is to bound the term involving the advantage function.

\textbf{Bounding the Advantage Function:}

First, note that the total reward, including human feedback, is:
\[
R_k(s, a) = R_k^0(s, a) + \lambda H_k(s, a).
\]

The maximum possible total reward is:
\[
R_{\text{total}, \max} = R_{\text{max}} + \lambda H_{\max}.
\]

The Q-function and Value function are bounded as:
\[
|Q_k^{\pi}(s, a)| \leq \frac{R_{\text{total}, \max}}{1 - \gamma}, \quad |V_k^{\pi}(s)| \leq \frac{R_{\text{total}, \max}}{1 - \gamma}.
\]

Therefore, the advantage function satisfies:
\[
|A_k^{\pi}(s, a)| = |Q_k^{\pi}(s, a) - V_k^{\pi}(s)| \leq \frac{2 R_{\text{total}, \max}}{1 - \gamma}.
\]

\textbf{Bounding the Expected Advantage:}

We aim to relate the expected advantage to the divergence between \(\pi_k\) and \(\pi\). Using the properties of probability distributions and the Cauchy-Schwarz inequality, we have:
\[
\left| \mathbb{E}_{a \sim \pi_k(\cdot|s,\theta)} \left[ A_k^{\pi}(s, a) \right] \right| = \left| \sum_{a} \pi_k(a|s,\theta) A_k^{\pi}(s, a) \right| = \left| \sum_{a} \left( \pi_k(a|s,\theta) - \pi(a|s,\theta) \right) A_k^{\pi}(s, a) \right|.
\]
\[
\leq \sum_{a} \left| \pi_k(a|s,\theta) - \pi(a|s,\theta) \right| \left| A_k^{\pi}(s, a) \right| \leq \frac{2 R_{\text{total}, \max}}{1 - \gamma} \sum_{a} \left| \pi_k(a|s,\theta) - \pi(a|s,\theta) \right|.
\]
\[
= \frac{4 R_{\text{total}, \max}}{1 - \gamma} D_{\text{TV}}\left( \pi_k(\cdot|s,\theta), \pi(\cdot|s,\theta) \right).
\]

\textbf{Relating Total Variation to KL Divergence:}

Using Pinsker's inequality, which relates the Total Variation distance to the Kullback-Leibler (KL) divergence:
\[
D_{\text{TV}}\left( \pi_k(\cdot|s,\theta), \pi(\cdot|s,\theta) \right) \leq \sqrt{\frac{1}{2} D_{\text{KL}}\left( \pi_k(\cdot|s,\theta) \parallel \pi(\cdot|s,\theta) \right) }.
\]

Substituting back:
\[
\left| \mathbb{E}_{a \sim \pi_k(\cdot|s,\theta)} \left[ A_k^{\pi}(s, a) \right] \right| \leq \frac{4 R_{\text{total}, \max}}{ (1 - \gamma) } \sqrt{\frac{1}{2} D_{\text{KL}}\left( \pi_k(\cdot|s,\theta) \parallel \pi(\cdot|s,\theta) \right) }.
\]

\textbf{Bounding the Expectation over State Distribution:}

Taking expectation over \( s \sim d^{\pi_k}_k(s) \) and applying Jensen's inequality:
\[
\left| \mathbb{E}_{s \sim d^{\pi_k}_k} \left[ \mathbb{E}_{a \sim \pi_k(\cdot|s,\theta)} \left[ A_k^{\pi}(s, a) \right] \right] \right| \leq \frac{4 R_{\text{total}, \max}}{ (1 - \gamma) } \sqrt{\frac{1}{2} \mathbb{E}_{s \sim d^{\pi_k}_k} \left[ D_{\text{KL}}\left( \pi_k(\cdot|s,\theta) \parallel \pi(\cdot|s,\theta) \right) \right] }.
\]

\textbf{Relating to Personalization Score:}

The personalization score \( P_k(\theta) \) (Definition~\ref{def:personalization_score}) is defined as:
\[
P_k(\theta) = \int \rho(s) D_{\text{KL}}\left( \pi_k(\cdot|s,\theta) \parallel \pi(\cdot|s,\theta) \right) ds.
\]

To relate the expectation over \( d^{\pi_k}_k(s) \) to \( P_k(\theta) \), we use the following approach based on the variational representation of KL divergence:
\[
\mathbb{E}_{s \sim d^{\pi_k}_k} \left[ D_{\text{KL}}\left( \pi_k(\cdot|s,\theta) \parallel \pi(\cdot|s,\theta) \right) \right] = \int d^{\pi_k}_k(s) D_{\text{KL}}\left( \pi_k(\cdot|s,\theta) \parallel \pi(\cdot|s,\theta) \right) ds.
\]
\[
= \int \frac{d^{\pi_k}_k(s)}{\rho(s)} \rho(s) D_{\text{KL}}\left( \pi_k(\cdot|s,\theta) \parallel \pi(\cdot|s,\theta) \right) ds.
\]
\[
\leq \left(\max_s \frac{d^{\pi_k}_k(s)}{\rho(s)}\right) \int \rho(s) D_{\text{KL}}\left( \pi_k(\cdot|s,\theta) \parallel \pi(\cdot|s,\theta) \right) ds.
\]
\[
= \max_s \frac{d^{\pi_k}_k(s)}{\rho(s)} \cdot P_k(\theta).
\]

\textbf{Bounding the Divergence \( D_{\infty} \):}

To formalize the boundedness of \( D_{\infty}(d^{\pi_k}_k \parallel \rho) \), we leverage the properties of the FedRLHF algorithm:

\begin{itemize}
    \item \textbf{Federated Aggregation:} The periodic aggregation of client models into a global model ensures that individual client policies do not diverge significantly from the global policy. This inherently limits the divergence between their respective state distributions.
    
    \item \textbf{Assumptions on MDP Properties:} The existing assumptions on smoothness and bounded gradients (Assumptions~\ref{assumption:L-smoothness} and~\ref{assumption:bounded-human-feedback}) further constrain how much the policies can alter state visitation distributions.
\end{itemize}

Given these factors, we can assert that \( D_{\infty}(d^{\pi_k}_k \parallel \rho) \) is bounded by a small constant, specifically \( D_{\infty}(d^{\pi_k}_k \parallel \rho) \leq \log 2 \). This implies:
\[
\exp\left(D_{\infty}(d^{\pi_k}_k \parallel \rho)\right) \leq 2.
\]
This is a reasonable assumption in practical FedRLHF implementations where policies are regularly synchronized and updated to prevent excessive divergence.

\textbf{Final Bound:}

Substituting the bounded divergence into the previous inequality:
\[
\mathbb{E}_{s \sim d^{\pi_k}_k} \left[ D_{\text{KL}}\left( \pi_k(\cdot|s,\theta) \parallel \pi(\cdot|s,\theta) \right) \right] \leq 2 P_k(\theta).
\]

Therefore,
\[
\left| \mathbb{E}_{s \sim d^{\pi_k}_k} \left[ \mathbb{E}_{a \sim \pi_k(\cdot|s,\theta)} \left[ A_k^{\pi}(s, a) \right] \right] \right| \leq \frac{4 R_{\text{total}, \max}}{ (1 - \gamma) \sqrt{2} } \sqrt{ 2 P_k(\theta) } = \frac{4 R_{\text{total}, \max}}{ (1 - \gamma) } \sqrt{ \frac{P_k(\theta)}{2} }.
\]

Simplifying,
\[
\left| \mathbb{E}_{s \sim d^{\pi_k}_k} \left[ \mathbb{E}_{a \sim \pi_k(\cdot|s,\theta)} \left[ A_k^{\pi}(s, a) \right] \right] \right| \leq \frac{4 R_{\text{total}, \max}}{ (1 - \gamma)^2 \sqrt{2} } \sqrt{ P_k(\theta) }.
\]

Substituting back into the expression for \( J_k^0(\pi) \):
\[
J_k^0(\pi) \geq J_k^0(\pi_k) - \frac{4 R_{\text{total}, \max}}{ (1 - \gamma)^2 \sqrt{2} } \sqrt{ P_k(\theta) }.
\]

Averaging over all clients:
\[
J_g(\theta) = \frac{1}{K}\sum_{k=1}^K J_k^0(\pi) \geq \frac{1}{K}\sum_{k=1}^K J_k^0(\pi_k) - \frac{4 R_{\text{total}, \max}}{ (1 - \gamma)^2 \sqrt{2} } \cdot \frac{1}{K}\sum_{k=1}^K \sqrt{ P_k(\theta) }.
\]
\[
= \frac{1}{K}\sum_{k=1}^K J_k^0(\pi_k) - C \cdot \left( \frac{1}{K}\sum_{k=1}^K \sqrt{ P_k(\theta) } \right),
\]
where
\[
C = \frac{4 R_{\text{total}, \max}}{ (1 - \gamma)^2 \sqrt{2} } = \frac{2 \sqrt{2} R_{\text{total}, \max}}{ (1 - \gamma)^2 }.
\]

This completes the proof. 

\end{proof}

\begin{theorem}[[Restatement of Theorem~\ref{theorem:impact_human_feedback} Impact of Human Feedback]
Under the same assumptions and definitions in Theorem~\ref{theorem:personalization_performance_tradeoff}, as the human feedback weight \( \lambda \) increases:
\begin{enumerate}
    \item The average personalization score \( \frac{1}{K}\sum_{k=1}^K P_k(\theta) \) increases at a rate of \( O(\lambda^2) \).
    \item The global performance \( J_g(\theta) \) decreases at a rate of \( O(\lambda) \).
    \item The sample complexity \( N \) increases at a rate of \( O(\lambda) \).
\end{enumerate}
\end{theorem}

\begin{proof}[Proof]
\textbf{Effect on Personalization Score:}

The incorporation of human feedback modifies the policy updates for each client. Specifically, the gradient of the human feedback term influences the policy parameters. Assume that the policy update rule for client \( k \) is given by:
\begin{equation*}
\theta_k' = \theta_k + \eta \lambda \nabla_\theta H_k(\theta_k),
\end{equation*}
where \( \eta \) is the learning rate, \( \lambda \) is the human feedback weight, and \( \nabla_\theta H_k(\theta_k) \) is the gradient of the human feedback with respect to the policy parameters.

Under Assumption~\ref{assumption:G-bounded-gradients}, we assume that \( \| \nabla_\theta H_k(\theta_k) \| \leq G_H \) for some constant \( G_H \).

Using the smoothness of the policy (Assumption~\ref{assumption:L-smoothness}), the KL divergence between the updated client policy and the global policy can be bounded as:
\begin{align*}
D_{\text{KL}}(\pi_k(\cdot|s,\theta_k') \parallel \pi(\cdot|s,\theta)) &\leq \frac{L}{2} \| \theta_k' - \theta \|^2,
\end{align*}
where \( L \) is the Lipschitz constant for the gradients.

Substituting the update rule:
\begin{align*}
\| \theta_k' - \theta \|^2 &= \| \theta_k - \theta + \eta \lambda \nabla_\theta H_k(\theta_k) \|^2 \\
&\leq 2 \| \theta_k - \theta \|^2 + 2 (\eta \lambda)^2 \| \nabla_\theta H_k(\theta_k) \|^2 \\
&\leq 2 \| \theta_k - \theta \|^2 + 2 (\eta \lambda G_H)^2,
\end{align*}
where the inequality follows from the Cauchy-Schwarz inequality.

Assuming that before the update, the policies are close, i.e., \( \| \theta_k - \theta \| \) is small (which is reasonable due to federated averaging), the dominant term is \( (\eta \lambda G_H)^2 \). Therefore:
\begin{align*}
\| \theta_k' - \theta \|^2 &\leq 2 (\eta \lambda G_H)^2.
\end{align*}

Substituting back into the KL divergence bound:
\begin{align*}
D_{\text{KL}}(\pi_k(\cdot|s,\theta_k') \parallel \pi(\cdot|s,\theta)) &\leq \frac{L}{2} \cdot 2 (\eta \lambda G_H)^2 \\
&= L (\eta \lambda G_H)^2.
\end{align*}

Taking expectation over \( s \sim \rho \), the personalization score \( P_k(\theta_k') \) satisfies:
\begin{align*}
P_k(\theta_k') &= \mathbb{E}_{s \sim \rho} \left[ D_{\text{KL}}\left(\pi_k(\cdot|s,\theta_k') \parallel \pi(\cdot|s,\theta)\right) \right] \\
&\leq L (\eta \lambda G_H)^2.
\end{align*}

Therefore, the personalization score scales as:
\begin{align*}
P_k(\theta_k') &= O(\lambda^2).
\end{align*}

Averaging over all clients:
\begin{align*}
\frac{1}{K}\sum_{k=1}^K P_k(\theta) &= O(\lambda^2).
\end{align*}

\textbf{Effect on Global Performance:}

From Theorem~\ref{theorem:personalization_performance_tradeoff}, the global performance satisfies:
\begin{equation*}
J_g(\theta) \geq \frac{1}{K}\sum_{k=1}^K J_k^0(\pi_k) - C \cdot \left( \frac{1}{K}\sum_{k=1}^K \sqrt{P_k(\theta)} \right),
\end{equation*}
where $C = \frac{2 \sqrt{2} R_{\text{total}, \max}}{(1 - \gamma)^2}$.

Given that \( P_k(\theta) = O(\lambda^2) \), we have \( \sqrt{P_k(\theta)} = O(\lambda) \). Therefore, the penalty term scales as:
\begin{align*}
C \cdot \left( \frac{1}{K}\sum_{k=1}^K \sqrt{P_k(\theta)} \right) &= O(\lambda).
\end{align*}

Thus, the global performance \( J_g(\theta) \) decreases at a rate of \( O(\lambda) \):
\begin{align*}
J_g(\theta) &\geq \frac{1}{K}\sum_{k=1}^K J_k^0(\pi_k) - O(\lambda).
\end{align*}

\textbf{Effect on Sample Complexity:}

From Theorem~\ref{theorem:sample-complexity}, the total number of samples required to achieve an expected optimality gap of \( \epsilon \) is:
\begin{equation*}
N = O\left( \frac{L (G^2 + \sigma^2)}{\mu^2 \epsilon^2} \right),
\end{equation*}
subject to:
\begin{align*}
K &\geq O\left( \frac{G^2 + \sigma^2}{\mu \epsilon} \right), \\
\lambda H_{\max} &\leq O\left( \frac{\mu \epsilon}{L} \right).
\end{align*}

Here, the term \( \frac{L K \lambda H_{\max}}{ \mu \epsilon } \) in the sample complexity indicates that \( N \) scales linearly with \( \lambda \):
\begin{align*}
N &= O\left( \frac{L (G^2 + \sigma^2)}{\mu^2 \epsilon^2} + \frac{L K \lambda H_{\max}}{ \mu \epsilon } \right) \\
&= O\left( \frac{L (G^2 + \sigma^2)}{\mu^2 \epsilon^2} + \frac{L K \lambda H_{\max}}{ \mu \epsilon } \right).
\end{align*}

Assuming that \( \lambda \) directly influences the second term, the overall sample complexity \( N \) increases at a rate of \( O(\lambda) \).
\end{proof}

\section{Full Experiment Details}\label{appendix:sec:full-experiments-details}
In this section, we provide a comprehensive overview of the experiments conducted to evaluate the effectiveness of FedRLHF in integrating human feedback within a federated reinforcement learning setting. We compare the performance of FedRLHF against a centralized RLHF baseline, highlighting the advantages of our approach in terms of privacy preservation, personalization, and maintaining the performance metrics of conventional RLHF methods.


\begin{enumerate}
    \item \textbf{Movie Rating Prediction on MovieLens:} In this task, we simulate a scenario where a streaming service provider seeks to enhance its recommendation system by learning from user interactions while preserving privacy and catering to individual preferences. Each client represents a user with unique viewing histories and preferences, introducing inherent data heterogeneity.
    \item \textbf{Sentiment-Controlled Movie Review Generation on IMDb:} In this task, we simulate a scenario where multiple movie review platforms collaborate to fine-tune a language model for sentiment-controlled text generation without centralizing their proprietary data or user feedback. Each client represents a distinct platform with its own collection of movie reviews and user interactions, introducing natural data heterogeneity. The goal is to train a model capable of generating movie reviews with controlled sentiment, while preserving the privacy of each platform's data and catering to their individual content preferences.
\end{enumerate}

All experiments were conducted on a machine equipped with a single NVIDIA GeForce RTX 3090 GPU with 24GB VRAM. We employed the Flower framework~\cite{beutel2020flower} to simulate a realistic federated learning environment, utilizing the gRPC protocol for communication to mimic the distributed nature of real-world federated systems.

\subsection{Movie Rating Prediction on MovieLens}

\subsubsection{Experimental Setup}

We utilized the \texttt{ml-latest-small} version of the MovieLens dataset~\cite{harper2015movielens}, which comprises 100,836 ratings from 610 users on 9,742 movies. For our experiments, we randomly selected $K=10$ users from the dataset to act as clients, ensuring a diverse representation of viewing histories and preferences. 
This setup introduces heterogeneity as each client possesses unique interaction patterns.
The objective was to predict whether a user would assign a high rating (4 stars or above) to a given movie, effectively framing this as a binary classification task.

\subsubsection{Human Feedback Simulation}\label{appendix:subsec:human-feedback-details}

To emulate realistic user behavior and feedback mechanisms, we developed a noise-aware, rule-based feedback simulator for each client. In real-world scenarios, human feedback is often sparse, noisy, and heterogeneous. Our simulator generates two types of feedback:

\begin{enumerate}
    \item \textbf{Direct Feedback:} This feedback categorizes model predictions as ``too high,'' ``too low,'' or ``about right.'' Specifically:
    \begin{itemize}
        \item \textbf{Too High:} If the predicted rating exceeds the user's actual rating by more than 0.5 points (after adding Gaussian noise to the true rating), the feedback is assigned a value of $-1$.
        \item \textbf{Too Low:} If the predicted rating is more than 0.5 points below the noisy true rating, the feedback is assigned a value of $1$.
        \item \textbf{About Right:} If the predicted rating is within 0.5 points of the noisy true rating, the feedback is assigned a value of $0$.
    \end{itemize}
    
    \item \textbf{Comparative Feedback:} This feedback expresses preferences between pairs of movies by comparing their noisy ground truth ratings. For each pair, the simulator:
    \begin{enumerate}
        \item Compares the true ratings of the two movies.
        \item Determines the true preference:
        \begin{itemize}
            \item Assigns $1$ if the first movie is preferred over the second.
            \item Assigns $-1$ if the second movie is preferred over the first.
            \item Assigns $0$ if both movies are equally rated.
        \end{itemize}
        \item Compares the true preference with the model's predicted preference.
        \item Returns the true preference if it differs from the predicted preference; otherwise, returns $0$ to indicate a correct prediction.
    \end{enumerate}
    
    The rationale behind incorporating comparative feedback is based on the observation that users often find it easier and more intuitive to express preferences between options rather than providing absolute ratings. This approach mirrors real-world scenarios where users might be inconsistent in assigning numerical ratings but can consistently indicate their preference between two items.
\end{enumerate}

The feedback values are bounded within $-1$ and $1$, satisfying Assumption~\ref{assumption:bounded-human-feedback} regarding bounded human feedback, which ensures stability in the learning process as discussed in Section \ref{sec:problem_formulation}.
The feedback collected is used to train a local \emph{reward model} on each client. This neural network takes as input the state (comprising user and movie IDs) and outputs a scalar predicted reward. The policy model is updated using Q-learning, with Q-values influenced by the rewards predicted by the reward model. This setup enables the policy to align recommendations with user preferences effectively.

\subsubsection{Implementation Details}\label{appendix:subsec:movielens-implementation-details}

We implemented a neural network-based model comprising embedding layers for users and movies, each with an embedding dimension of 50. This was followed by two fully connected layers with sizes [100, 50] and ReLU activation functions, and a sigmoid output layer for binary classification. The input features to the model included user IDs, movie IDs, and movie genre information, allowing the model to capture complex interactions between users and movies.

For the federated learning process, each client trained the model locally using a combination of intrinsic rewards and simulated human feedback, employing a Q-learning algorithm as the local RLHF step. Specifically, clients performed 5 local epochs of training per federated communication round, utilizing the Adam optimizer with a learning rate of $1 \times 10^{-3}$. 

Global model aggregation was conducted using a weighted average based on the number of examples per client, following a variant of the FedAvg algorithm~\cite{mcmahan2017communication}. The federated training process was carried out over a total of 5 communication rounds.


\begin{figure}
    \centering
    \includegraphics[width=\linewidth]
    {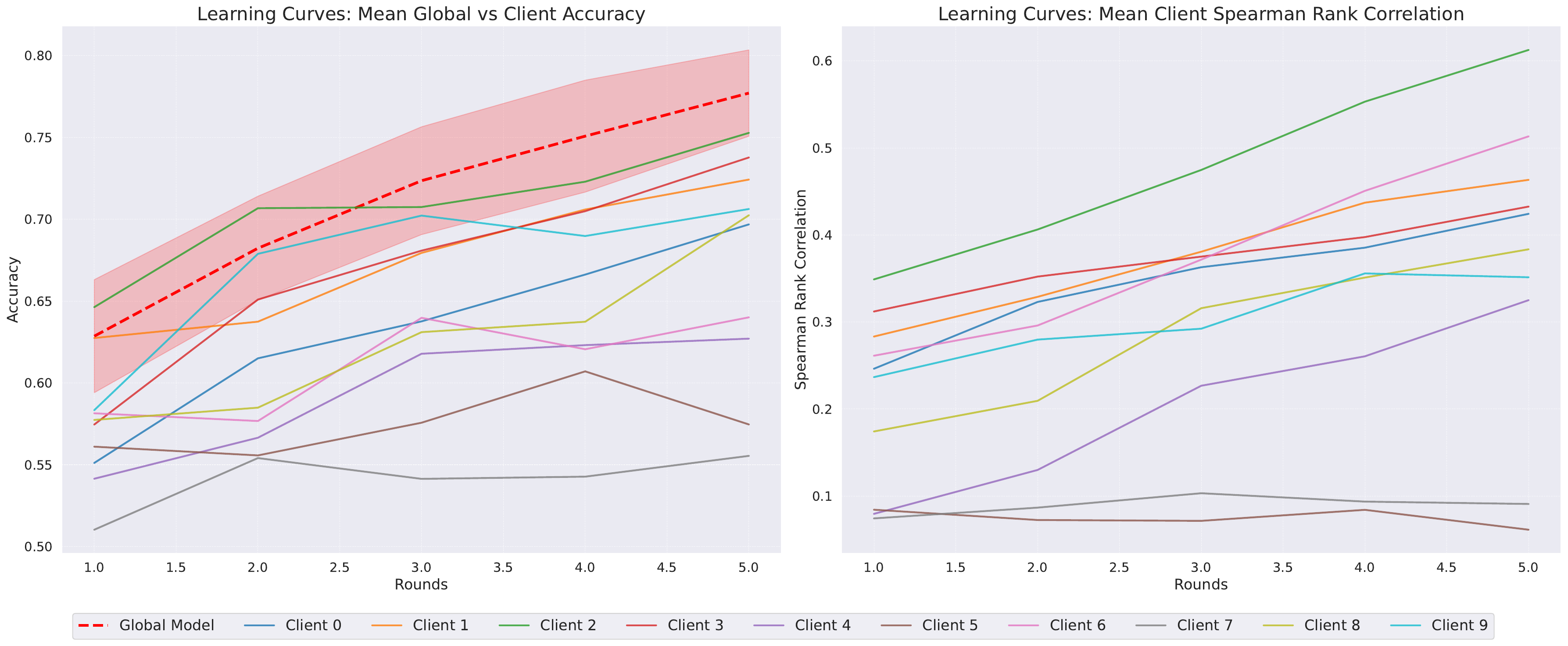}
    \caption{Learning curves for the MovieLens task: (a) Global vs. Client Accuracy, (b) Client Spearman Rank Correlation}
    \label{fig:movielens-performance-app}
        \Description{A plot showing the learning curves of client accuracy and Spearman rank correlation for the MovieLens dataset. The global accuracy improves steadily across rounds.}
\end{figure}
The following hyperparameters were used for training in the MovieLens task:


\subsubsection{Results and Analysis}

Figure \ref{fig:movielens-performance-app} displays the learning curves for both global and client-specific accuracies, as well as the Spearman rank correlations for each of the $K=10$ clients over the course of 5 federated rounds. The global accuracy, represented by the red dotted line, is the weighted average of the client accuracies. The results are averaged over five independent runs with random seeds 0, 42, 101, 123, and 4242 to ensure statistical significance. 

\begin{figure}
    \centering
    \includegraphics[width=0.8\linewidth]{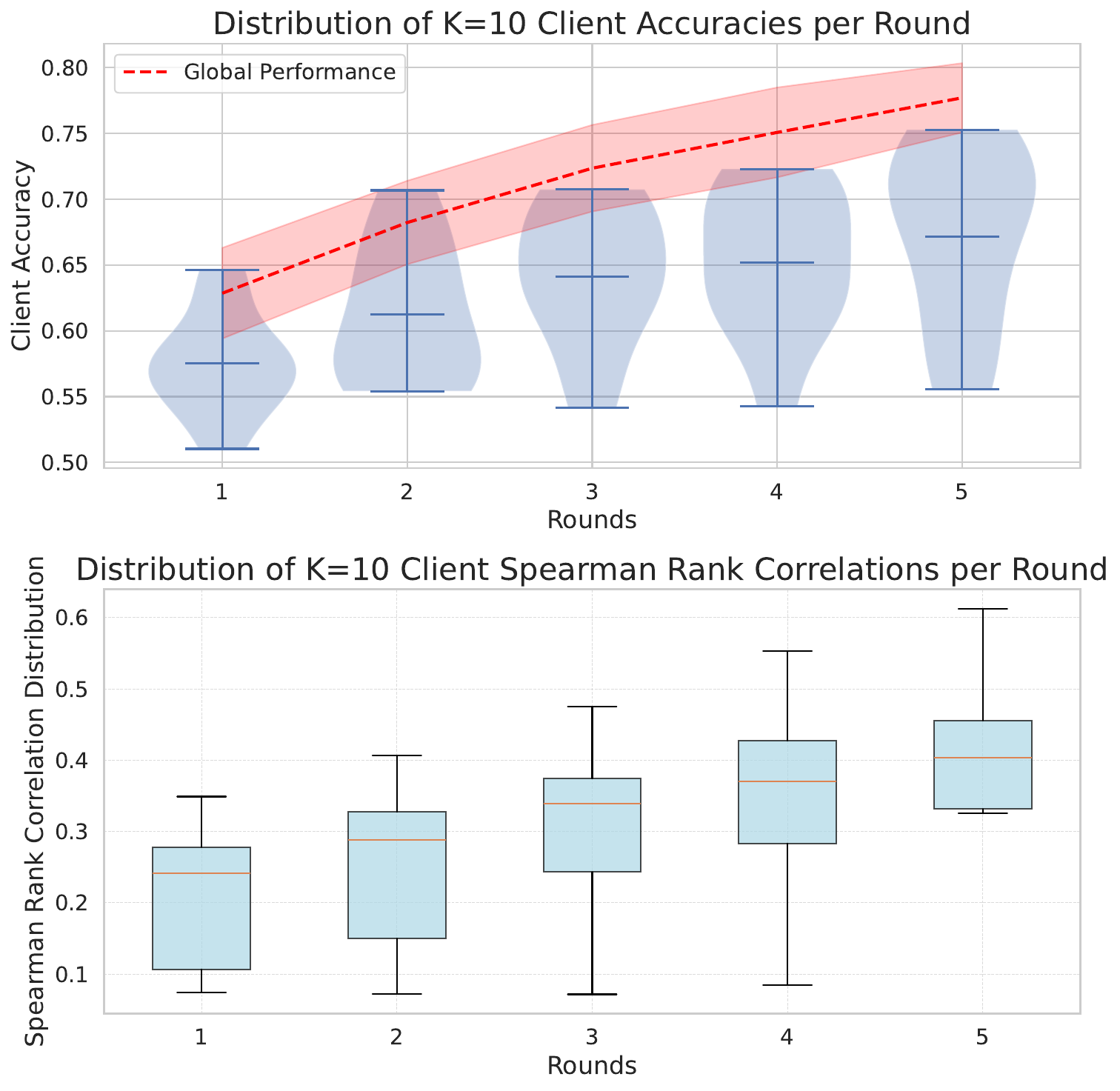}
    \caption{Distribution of client accuracies and Spearman rank correlations per round for the MovieLens task with $K=10$ clients.}
    \label{fig:movielens-distribution-k10-app}
    \Description{Violin and box plots showing the distribution of client accuracies and Spearman rank correlations for the MovieLens task with $K=10$ clients.}
\end{figure}

\paragraph{Global Performance Improvements}

The global performance demonstrated steady improvement in accuracy.
Specifically, the increase from $62.86\% \pm 3.45\%$ to $77.71\% \pm 2.64\%$ over 5 rounds demonstrates the algorithm's convergence rate, which aligns with the $O(1/T)$ rate established in Theorem~\ref{theorem:convergence}.
%
The violin plots in Figure~\ref{fig:movielens-distribution-k10-app} (top subplot) reveal a notable dispersion in individual client accuracies, spanning from $55.54\%$ (Client 7) to $75.28\%$ (Client 2) in the concluding round. This variability is indicative of the inherent heterogeneity among clients, stemming from differences in data quality, unique user preferences, and varying interaction patterns. Such diversity emphasizes the framework's capacity to accommodate personalized learning while maintaining robust global performance.


\paragraph{Personalization-Performance Trade-off}


To better understand the personalization effect, we employed Spearman rank correlation as a metric for evaluating how well the model captured the relative ordering of movies according to user preferences. This metric, unlike accuracy, is particularly relevant for recommendation systems where appropriate ranking is often more crucial than exact rating prediction. 
This serves as a useful surrogate for the personalization-performance trade-off discussed in Theorem~\ref{theorem:personalization_performance_tradeoff}, where personalized models enhance client-specific performance at the expense of global alignment.
Our findings reveal substantial variability in Spearman correlations across clients, ranging from $0.0613$ (Client 5) to $0.6126$ (Client 2) (Figure~\ref{fig:movielens-distribution-k10}b). High correlations, observed in Clients 1, 2, 6, and 8, signify effective personalization, resulting in recommendation rankings that closely mirror user-specific preferences. Conversely, low correlations in Clients 5 and 7 suggest challenges in aligning with these users' nuanced preferences, potentially attributable to factors such as data sparsity or inherently complex preference structures.
%
%
Moreover, the violin plots (Figure~\ref{fig:movielens-distribution-k10}b) demonstrate an upward trend in median Spearman rank correlations across federated rounds. This progression indicates that as training advances, the FedRLHF framework increasingly facilitates the development of personalized models that adeptly align with individual user preferences. 


\paragraph{Scaling up to $K=50$ Clients}\label{appendix:para:movielens-results-k50}
We observed similar trends when scaling up the experiment to $K=50$ clients. The learning curves for individual clients and global performance are shown in Figure~\ref{fig:movielens-performance-k50}. The global accuracy (red dashed line) demonstrates steady improvement over the rounds, while individual client accuracies exhibit significant variability. Figure~\ref{fig:movielens-distribution-k50} provides violin plots showing the distribution of client accuracies and Spearman rank correlations across rounds. The client accuracy distribution widens over rounds, indicating increased personalization, with the global performance consistently improving. The Spearman rank correlation distribution shows a general upward trend, with the median and upper quartile increasing across rounds, suggesting improved alignment between client-specific rankings and true preferences for many clients.

\begin{figure}[t]
    \centering
    \includegraphics[width=0.5\linewidth]
    {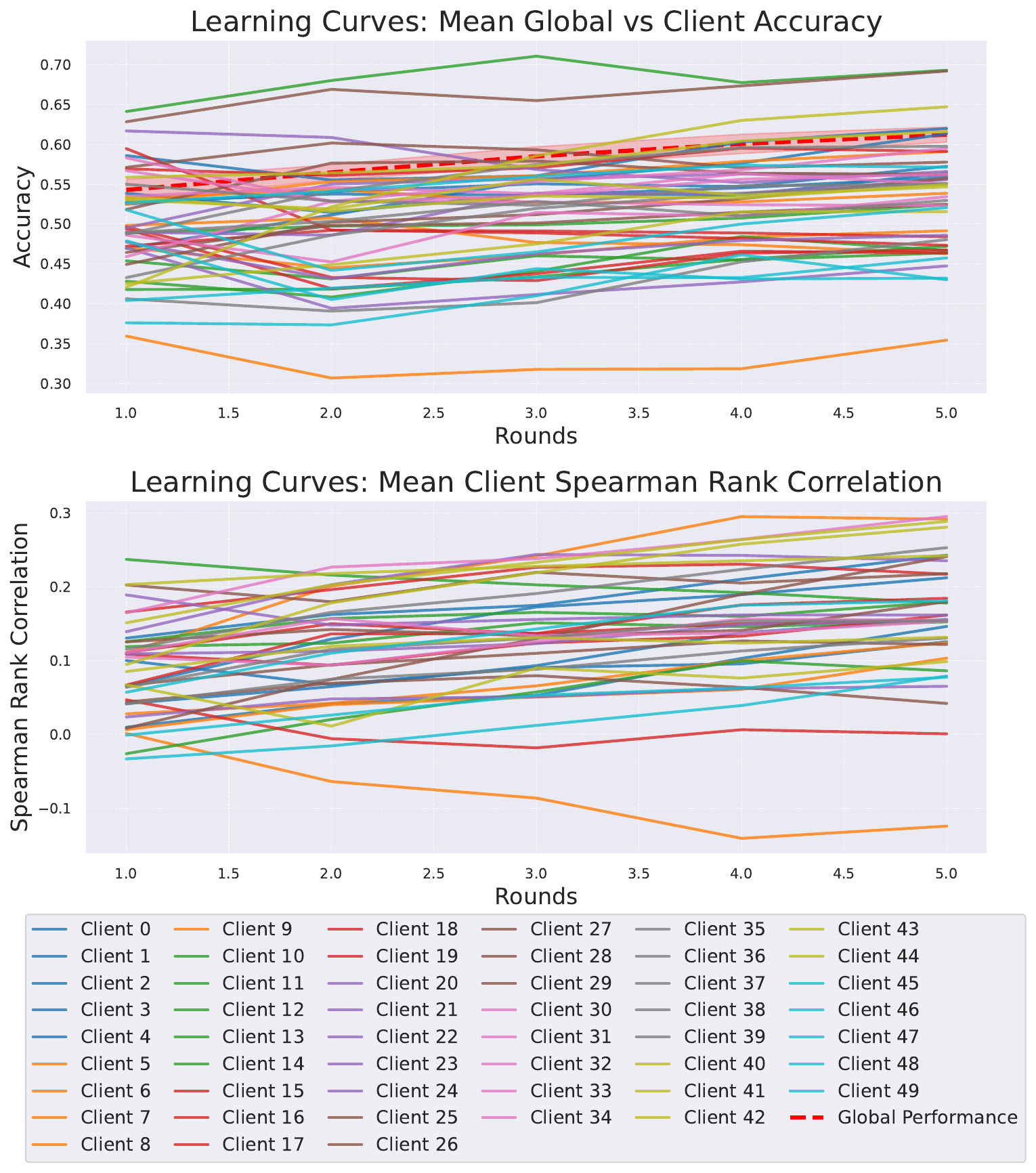}
    \caption{Learning curves on MovieLens: (top) Global vs. Client Accuracy, (bottom) Client Spearman correlation.}
    \label{fig:movielens-performance-k50}
        \Description{A plot showing the learning curves of client accuracy and Spearman correlation for the MovieLens dataset. The global accuracy improves steadily across rounds.}
\end{figure}

\begin{figure}
    \centering
    \includegraphics[width=0.8\linewidth]{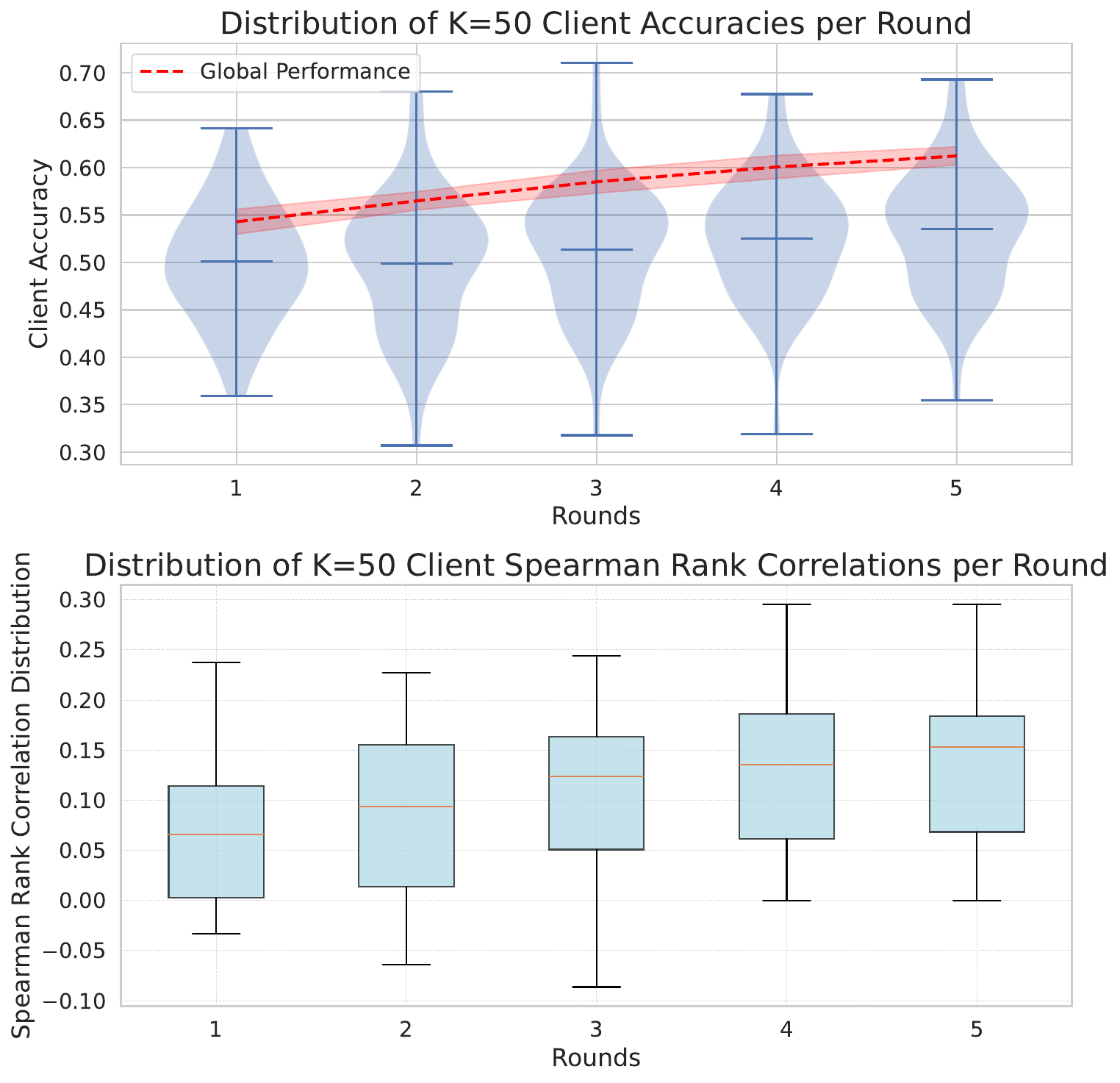}
    \caption{Distribution of client accuracies and Spearman rank correlations per round for the MovieLens task with $K=50$ clients.}
    \label{fig:movielens-distribution-k50}
    \Description{Violin and box plots showing the distribution of client accuracies and Spearman rank correlations for the MovieLens task with $K=50$ clients.}
\end{figure}

\subsection{Sentiment-Controlled Movie Review Generation on IMDb}
\subsubsection{Experimental Setup}
In this task, we simulated a scenario where multiple movie review platforms collaborate to fine-tune a language model for sentiment-controlled text generation without sharing their proprietary data or user feedback. Each client represents a distinct platform with its own collection of movie reviews, introducing natural data heterogeneity.

We utilized the IMDb dataset~\cite{maas2011learning}, which contains 50,000 movie reviews labeled as positive or negative. 
For our experiments, we partitioned the dataset among  $K=5$ clients, with each client receiving a unique, non-overlapping subset of the data. Specifically, we divided the dataset such that each client was allocated approximately 10,000 reviews, ensuring that clients possess disjoint training datasets. This setup reflects real-world scenarios where different platforms have access to different datasets and cannot share data due to privacy concerns.

\subsubsection{Implementation Details}\label{appendix:subsec:IMDb_implementation_details}

We employed a GPT-2 language model~\cite{radford2019language} for text generation, fine-tuned using the Proximal Policy Optimization (PPO) algorithm~\cite{schulman2017proximal-PPO} within the TRL (Transformer Reinforcement Learning) library~\cite{vonwerra2022trl}. 
We preprocessed the dataset by filtering out reviews shorter than 200 characters and limiting the total number of samples to manage computational resources. The reviews were then tokenized for input to the model.

Key components:
\begin{itemize}
\item Model: GPT-2 (125M parameters) with an additional value head for PPO training~\cite{schulman2017proximal-PPO}, resulting in a total of about 137M parameters
\item Training: PPO with batch size 16, learning rate 1e-5
\item Generation settings: top-k=0, top-p=1.0, temperature=1.0
\item Federated Rounds: 5
\item The RLHF approach is closely following the implementation provided by the TRL tutorial~\cite{vonwerra2022trl} on \texttt{gpt2-sentiment} task notebook.
\end{itemize}

Each client conducted local RLHF training by integrating intrinsic rewards and simulated human feedback. The intrinsic reward was derived from the negative log-likelihood of the generated text, promoting fluency and coherence. Clients performed 3 local epochs of training per federated round, using the Adam optimizer with a learning rate of $1 \times 10^{-5}$.
Global model aggregation was performed using FedAvg~\cite{mcmahan2017communication}, and the federated training spanned 5 communication rounds. 

To ensure fair evaluation and prevent data leakage between training and testing phases, we split each client's allocated data into training and evaluation sets. Specifically, 80\% of each client's data was used for training, and the remaining 20\% formed the evaluation dataset. The rewards and losses reported are calculated on this separate evaluation dataset that is not used for training.

\subsubsection{Human Feedback Simulation}

Human feedback was simulated using a sentiment analysis model (DistilBERT~\cite{sanh2019distilbert} fine-tuned on IMDb) implemented locally on each client. The reward function combined the sentiment score of the generated text and the language model's log probability, weighted by a client-specific parameter $\lambda_k \in [0,1]$, reflecting individual preferences for positivity. 
The combined reward $R_k$ for Client $k$ was calculated as:
\[
R_k = \lambda_k \cdot R_{\text{sentiment}} + (1-\lambda_{k}) \cdot R^0_{k}
\]
where $R_{\text{sentiment}}$ is the sentiment alignment reward and $R^0_k$ is the intrinsic fluency reward.
This formulation aligns with the reward shaping in Equation~\eqref{eq:reward-shaping}, allowing clients to personalize the importance of sentiment alignment.

\subsubsection{Results and Analysis}\label{appendix:subsec:movielens-results}
We report the results using a single random seed (42) to maintain consistency across experiments.

\paragraph{Global and Client-Specific Performance Over Rounds}
Figure~\ref{fig:imdb-performance-app} illustrates the global and client-specific average rewards and losses over the communication rounds. The global average reward, represented by the dashed line, shows a steady improvement from approximately 0.51 to 0.68 across the five rounds, indicating effective aggregation of client updates. The corresponding loss curves show a rapid initial decrease followed by a gradual stabilization, with Client 0 exhibiting the most dramatic reduction from about 0.25 to near 0.

\begin{figure}
\centering
\includegraphics[width=0.8\linewidth, height=0.8\linewidth]{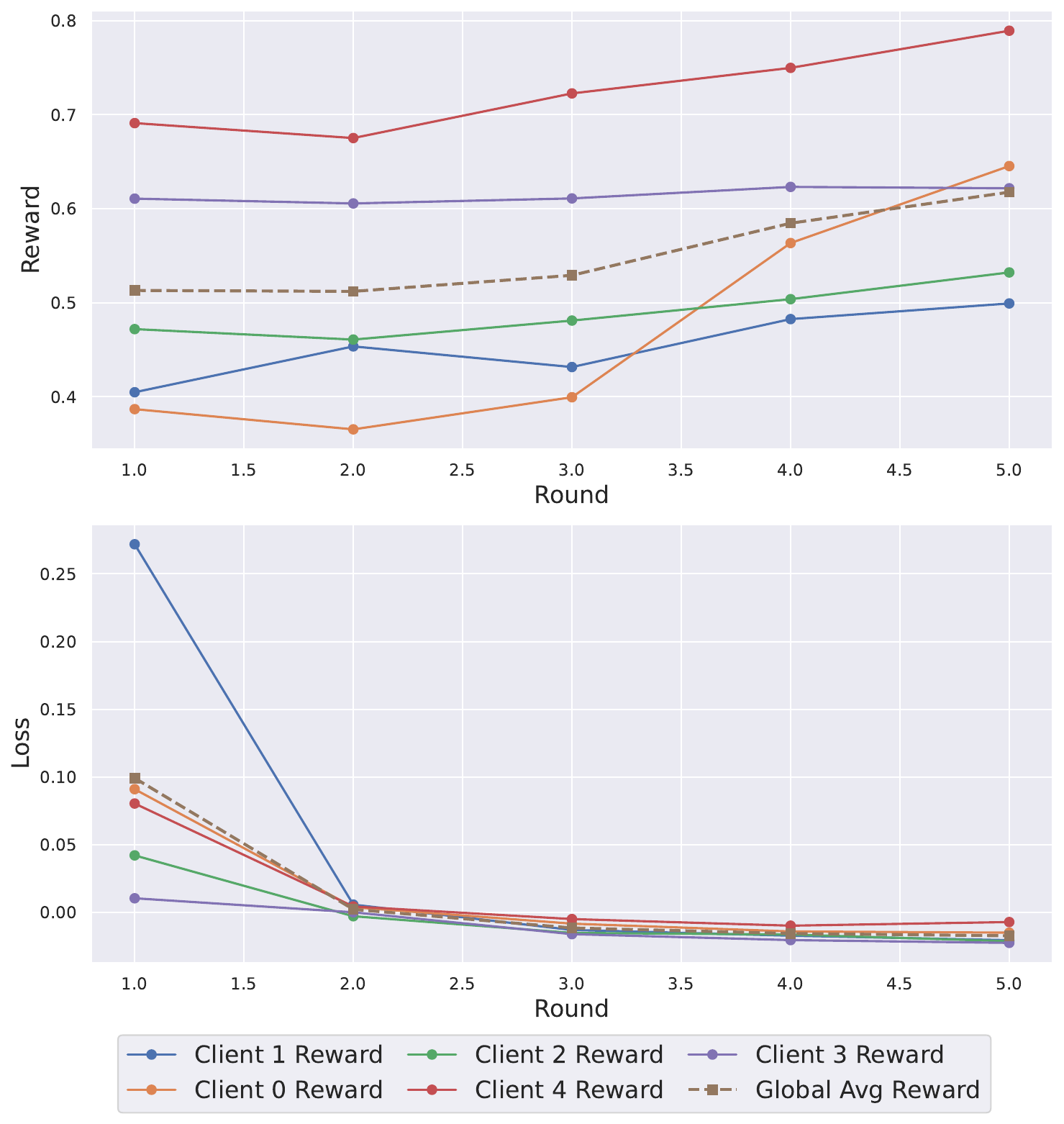}
\caption{Global and client-specific performance over communication rounds in the IMDb task.}
\Description{Global and client-specific performance over communication rounds in the IMDb task.}
\label{fig:imdb-performance-app}
\end{figure}

\begin{figure}
\centering
\includegraphics[width=.8\linewidth, height=0.8\linewidth]{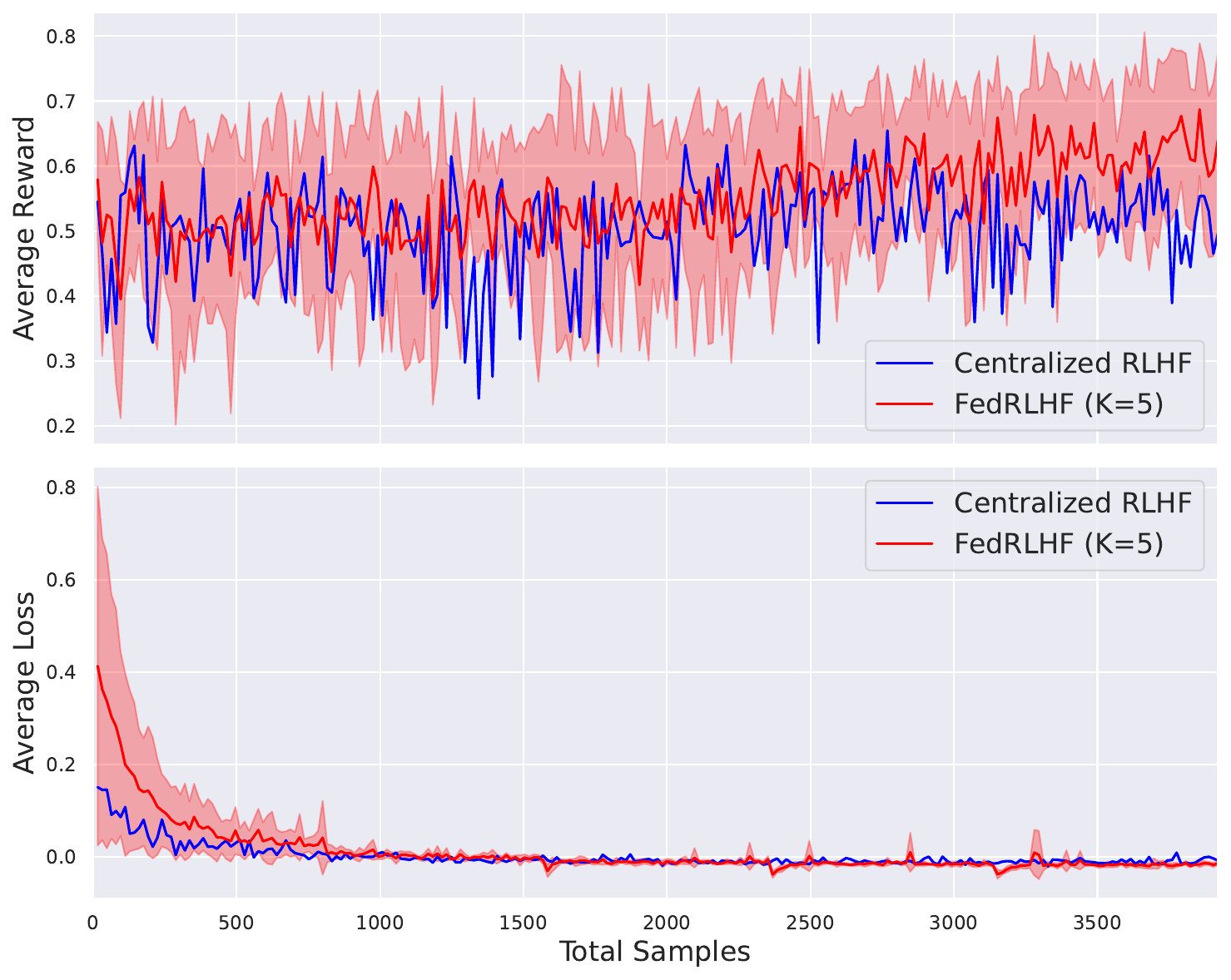}
\caption{Comparison of Average Rewards and Losses between Centralized RLHF and FedRLHF ($K=5$) over the total number of samples in the IMDb task.}
\Description{Comparison of Average Rewards and Losses between Centralized RLHF and FedRLHF ($K=5$) over the total number of samples in the IMDb task.}
\label{fig:imdb-performance_over_samples-app}
\end{figure}

\paragraph{Comparison with Centralized RLHF}
To further assess the performance of FedRLHF, we compared it with a centralized RLHF baseline trained on the aggregated data from all clients. Figure~\ref{fig:imdb-performance_over_samples-app} presents the average rewards and losses between the centralized RLHF and FedRLHF ($K=5$) over the total number of samples.
The rewards comparison reveals that while the centralized model initially achieves slightly higher rewards, FedRLHF quickly catches up and even surpasses the centralized model's performance in later stages. This is evident from the FedRLHF reward curve (in red) consistently lying above the centralized RLHF curve (in blue) after approximately 1500 samples.
This improvement arises from FedRLHF's ability to leverage diverse client data and regular model aggregation, which enhance generalization and reduce overfitting compared to the centralized approach.
The loss comparison shows that both approaches achieve similar loss reduction.
This result corroborates the sample complexity analysis in Theorem~\ref{theorem:sample-complexity}, indicating that FedRLHF can match or even exceed centralized performance while preserving privacy and distributing computation.


\paragraph{Personalization-Performance Trade-off}
To analyze the personalization-performance trade-off in FedRLHF, we conducted a detailed evaluation of how clients' personalized objectives affected their individual rewards over the training rounds. For each client, we randomly sampled 30 queries from their evaluation dataset at the beginning of training and kept these queries fixed throughout all rounds. In each communication round, we supplied these 30 queries to the client's model, recorded the generated responses, and calculated the corresponding intrinsic rewards ($R_{\text{intrinsic}}$), sentiment rewards ($R_{\text{sentiment}}$), and combined rewards ($R_k$).

Each client was assigned a different personalization weight $\lambda_k$, which controlled the emphasis on sentiment alignment versus intrinsic fluency. Specifically, the values of $\lambda_k$ were set to 0.1, 0.3, 0.5, 0.7, and 0.9 for clients 0, 1, 2, 3, and 4, respectively. This range of values allowed us to explore different levels of personalization, from minimal sentiment emphasis ($\lambda_k = 0.1$) to high sentiment emphasis ($\lambda_k = 0.9$).
Figure~\ref{fig:visualization_sentiment_int_rewards} presents the trends of intrinsic rewards ($R_{\text{intrinsic}}$), sentiment rewards ($R_{\text{sentiment}}$), and combined rewards ($R_k$) for each client over the communication rounds. The results reveal distinct patterns based on the personalization weights. For example, Client 0 ($\lambda_0 = 0.1$) prioritizes intrinsic rewards, exhibiting stable intrinsic and combined rewards while sentiment rewards fluctuate. In contrast, Client 2 ($\lambda_2 = 0.5$) shows a steady improvement in sentiment rewards while maintaining stable intrinsic rewards, reflecting an equal emphasis on both objectives. Client 4 ($\lambda_4 = 0.9$) demonstrates the most significant increase in sentiment rewards, accompanied by a noticeable decrease in intrinsic rewards, heavily emphasizing sentiment alignment. This improvement arises from FedRLHF's ability to leverage diverse client data and regular model aggregation, which enhance generalization and reduce overfitting compared to the centralized approach. 
The trends observed in the rewards over rounds indicate that the clients' models are effectively adapting to their personalized objectives. This adaptation aligns with our theoretical analysis (Theorem~\ref{theorem:personalization_performance_tradeoff}), where increasing the human feedback weight $\lambda$ enhances personalization (as indicated by improved personalized rewards) but may impact the global performance.

\begin{figure}[t]
    \centering
    \includegraphics[width=0.8\linewidth]{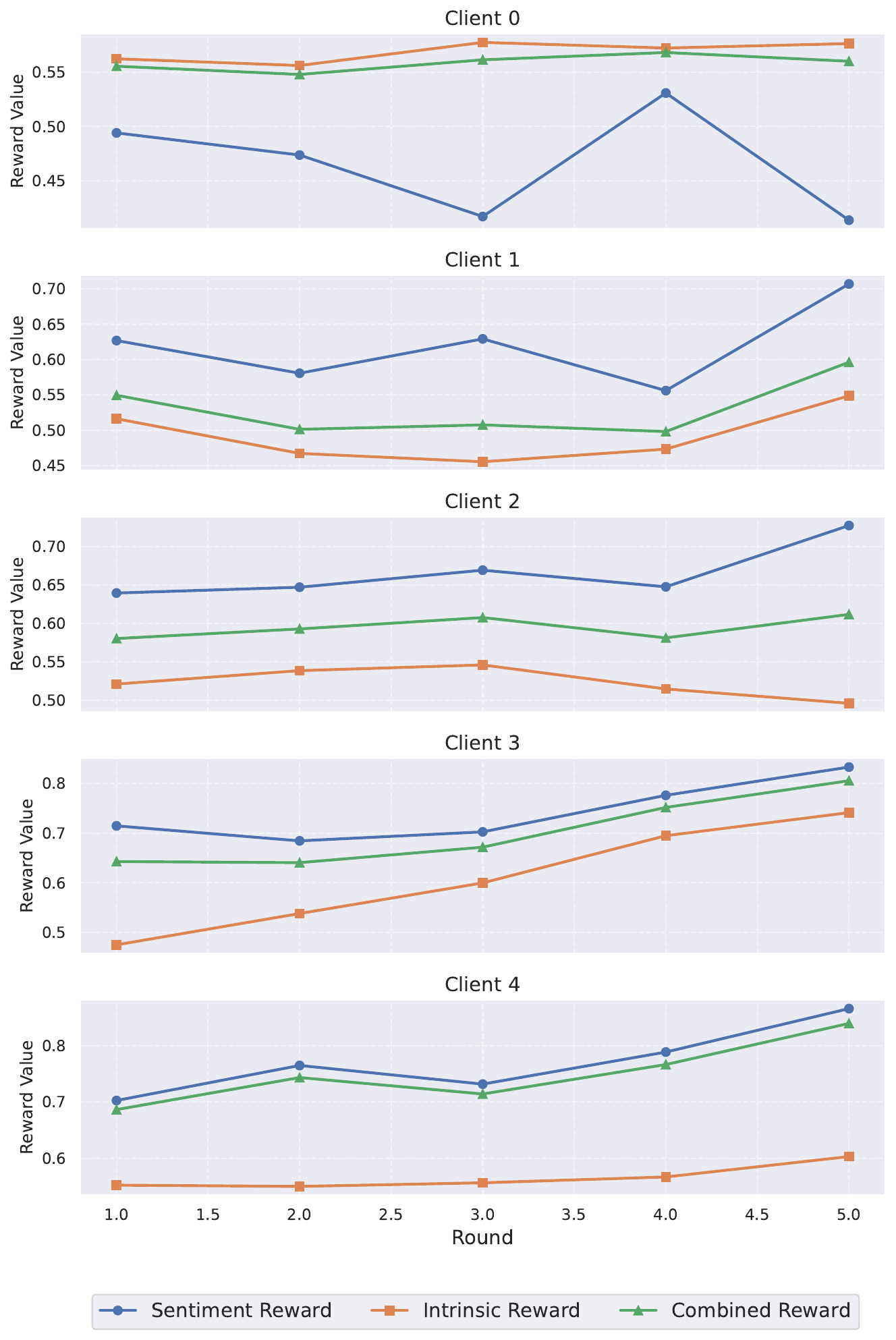}
    \caption{Trends of intrinsic rewards, sentiment rewards, and combined rewards over communication rounds for each client in the IMDb task. Each subplot corresponds to one client, illustrating personalization effects due to varying $\lambda_k$ values.}
        \Description{Trends of intrinsic rewards, sentiment rewards, and combined rewards over communication rounds for each client in the IMDb task. Each subplot corresponds to one client, illustrating personalization effects due to varying $\lambda_k$ values.}
    \label{fig:visualization_sentiment_int_rewards-app}
\end{figure}

\end{document}